\documentclass{article}

\usepackage{iclr2027_conference,times}

\iclrfinalcopy   

\usepackage[utf8]{inputenc} 
\usepackage[T1]{fontenc}    
\usepackage[table]{xcolor}
\definecolor{dark-blue}{RGB}{0,0,191}
\usepackage[hidelinks,colorlinks=true,citecolor=dark-blue,linkcolor=dark-blue,urlcolor=dark-blue, backref=page]{hyperref}
\usepackage{url}            
\usepackage{booktabs}       
\usepackage{amsfonts}       
\usepackage{nicefrac}       
\usepackage{microtype}      
\usepackage{xcolor}         

\title{Pessimistic Minimax Learning for Public-Private Information Games under Unilateral Coverage}

\usepackage{pifont}

\usepackage{amsmath}
\usepackage{amssymb}
\usepackage{mathtools}
\usepackage{mathrsfs}
\usepackage{amsthm}
\usepackage{color}
\mathtoolsset{showonlyrefs}

\usepackage{algorithm}
\usepackage{algpseudocode}

\usepackage{enumitem}

\usepackage[capitalize,noabbrev]{cleveref}

\usepackage{xurl} 
\usepackage{listings}
\usepackage{xcolor}

\usepackage{pifont}
\theoremstyle{plain}
\newtheorem{theorem}{Theorem}[section]
\newtheorem{proposition}[theorem]{Proposition}
\newtheorem{lemma}[theorem]{Lemma}

\theoremstyle{definition}

\newtheorem{assumption}[theorem]{Assumption}
\theoremstyle{remark}
\newtheorem{remark}[theorem]{Remark}

\newcommand{\V}{\mathbb{V}}

\DeclareMathOperator*{\argmin}{argmin}
\DeclareMathOperator*{\argmax}{argmax}

\newcommand{\field}[1]{\mathbb{#1}}

\newcommand{\E}{\field{E}}
\renewcommand{\P}{\field{P}}

\newcommand{\KL}{{\text{\rm KL}}}

\newcommand{\defeq}{\stackrel{\rm def}{=}}

\newcommand{\wh}{\widehat}

\newcommand{\Xcal}{\mathcal{X}}
\newcommand{\Acal}{\mathcal{A}}

\newcommand{\Bcal}{\mathcal{B}}
\newcommand{\Ccal}{\mathcal{C}}
\newcommand{\Dcal}{\mathcal{D}}
\newcommand{\Ecal}{\mathcal{E}}
\newcommand{\Fcal}{\mathcal{F}}
\newcommand{\Gcal}{\mathcal{G}}

\newcommand{\Lcal}{\mathcal{L}}

\newcommand{\Scal}{{\mathcal{S}}}

\newcommand{\Cscr}{\mathscr{C}}

\newcommand{\EE}{\mathbb{E}} 
\newcommand{\PP}{\mathbb{P}} 
\newcommand{\explain}[1]{\tag*{(#1)}}
\newcommand{\explaind}[2]{\makebox[0.85\textwidth]{$\displaystyle#1$\hfill(#2)}}
\usepackage[textsize=tiny]{todonotes}

\usepackage[normalem]{ulem}
\definecolor{claude-brown}{RGB}{133,45,15}
\definecolor{claude-green}{RGB}{0,100,0}

\author{%
 Shuze Daniel Liu \\
 Massachusetts Institute of Technology\\
  Purdue University\\
 \texttt{shuzel@mit.edu} \\
  \And
  Claire Chen\\
California Institute of Technology\\
  \texttt{clairechen@caltech.edu} \\
    \And
      Jiuqi Wang \\
  University of Virginia\\
  \texttt{jiuqi@email.virginia.edu} \\
  \And
    David Simchi-Levi \\
  Massachusetts Institute of Technology\\
  Purdue University\\
  \texttt{dslevi@mit.edu} \\
}

\begin{document}

\maketitle
\lhead{}   

\begin{abstract}
We study offline learning in two-player zero-sum contextual games with public and private information, motivated by strategic settings such as auctions and negotiations with private valuations. We introduce unilateral prescriptive concentrability and show that asymmetric information can change offline coverage through its effect on equilibrium behavior. For finite state-action spaces, we develop a pessimistic algorithm with an $\tilde{O}(1/\sqrt{n})$ exploitability rate, matching the standard sample-size dependence for fully observed minimax games. We further develop a pessimistic policy mirror descent framework, PPA-PMD, for general function approximation and obtain a unified $\tilde{O}(1/\sqrt{n} + 1/\sqrt{T})$ exploitability rate with no-regret actor updates. Together, these results provide the first theoretical framework for offline equilibrium learning under public-private information constraints.
\end{abstract}

\section{Introduction}

Offline reinforcement learning (RL) has emerged as a critical paradigm for developing autonomous agents in domains where real-time exploration is either prohibitively expensive, unethical, or safety-critical \citep{fujimoto2019off, levine2020offline}. By leveraging large-scale, pre-collected datasets, offline RL aims to recover optimal policies without further environmental interaction. However, the transition from logged data to optimal behavior is fundamentally challenged by distribution shift—the mismatch between the behavior policy that generated the data and the target policy being learned \citep{fujimoto2019off, levine2020offline, kumar2020conservative}. This shift often leads to extrapolation errors, where the learner overestimates the value of actions poorly represented in the dataset \citep{kumar2020conservative, wu2019behavior}. To address this, a significant body of literature has established the principle of pessimism as a provably efficient mechanism for handling distribution shift, typically through the use of lower confidence bounds (LCB) or conservative value estimates \citep{jin2021pessimism, xie2021bellman, rashidinejad2021bridging}.

While single-agent offline RL is well-studied, many high-stakes applications are best modeled as two-player zero-sum games \citep{shapley1953stochastic, littman1994markov}. In these competitive environments, the challenge of distribution shift is exacerbated by the need to identify a global minimax solution. Without suitable coverage, the equilibrium value can be unidentifiable, as the opponent’s best response can pivot toward unknown out-of-distribution regions. Recent advancements have addressed this by extending the principle of pessimism to two-player zero-sum games, utilizing lower confidence bounds to ensure that the learned policies are at least as good as the behavior policy across the supported regions \citep{zhong2022pessimistic, zhang2023offline}. However, these works typically assume a fully observed state space where both players possess a perfect view of the environment. In practice, strategic interactions—such as auctions or negotiations—often involve asymmetric observations where players hold private information hidden from their rivals \citep{vickrey1961counterspeculation, myerson1981optimal}.

Strategic interactions with hidden information are formally modeled as Bayesian Games \citep{harsanyi1967games} or games with Incomplete Information \citep{aumann1995repeated, kamien1990value}. In these settings, players observe a shared public context alongside their own private information, and they must act based on "prescriptions"—policies that condition on their local observations. This structure is a cornerstone of modern Operations Research (OR), particularly in the design of auctions \citep{milgrom1982theory, roughgarden2010algorithmic} and automated negotiation agents \citep{kraus2001strategic, jennings2001automated}. Despite its prevalence, the problem of learning optimal strategies for such games from offline data remains unexplored. The unique informational constraints of public-private games introduce a novel challenge: the learner must not only handle distribution shift but also manage the public belief—the joint distribution of private information given the public context, which players use to estimate the hidden state of their opponent.

To the best of our knowledge, this work provides the first theoretical analysis of offline two-player zero-sum contextual games under a public-private information structure. Our model builds on the classical contextual bandit setting \citep{dudik2011doubly, swaminathan2015batch} and captures applications such as auctions and negotiations with private valuations or costs \citep{jin2024sample}.
Information constraints change how players value actions under beliefs over hidden types, and therefore change which policy is an equilibrium. The coverage a dataset must provide is determined by the equilibrium the information structure induces, not by the size of the opponent's deviation class: at a common reference policy the prescriptive and fully observed coefficients coincide. We formalize this via \textit{unilateral prescriptive concentrability} and construct a family in which the induced equilibrium shift changes the coefficient of \citet{cui2022offline} by an arbitrarily large factor, in either direction.
Our contributions are summarized as follows:
\begin{itemize}[leftmargin=*]
    \item \textbf{Theoretical Foundations of Public-Private Coverage:} We introduce \emph{unilateral prescriptive concentrability} and show that public-private information can substantially change offline coverage by altering equilibrium behavior.
    \item \textbf{Optimal Learning in Finite Settings:} We propose \texttt{PPA-Exact}, a pessimistic algorithm for finite state-action spaces. We prove that it achieves a statistical exploitability rate of $\tilde{\mathcal{O}}(1/\sqrt{n})$, matching the sample-size dependence for fully observed minimax games~\citep{cui2022offline, zhong2022pessimistic}.
    \item \textbf{Scalable Optimization and Unified Complexity:} We propose \texttt{PPA-PMD}, an iterative pessimistic policy mirror descent framework for general function approximation that respects the game's inherent information constraints. We establish a unified exploitability gap of $\tilde{\mathcal{O}}(1/\sqrt{n} + 1/\sqrt{T})$, characterizing how statistical estimation from $n$ samples and optimization regret over $T$ iterations jointly determine the final performance in asymmetric strategic environments.
\end{itemize}

\section{Preliminaries}
\label{sec:preliminaries}
We study a two-player zero-sum contextual game with a public-private information structure. The environment is defined by the tuple $\mathcal{M} = \langle \mathcal{C}, \mathcal{X}_1, \mathcal{X}_2, \mathcal{A}_1, \mathcal{A}_2, r, \rho \rangle$, where the public context space $\mathcal{C}$, the private type spaces $\mathcal{X}_1, \mathcal{X}_2$, and the action spaces $\mathcal{A}_1, \mathcal{A}_2$ are finite. In each episode, the environment generates a latent state $s \in \mathcal{S}$ drawn from a fixed, unknown prior distribution $\rho \in \Delta(\mathcal{S})$. The latent state factors into a public context and two private types:
$s = (c, x^1, x^2) \in \mathcal{S} \defeq \mathcal{C} \times \mathcal{X}_1 \times \mathcal{X}_2,$
where $c \in \mathcal{C}$ is a shared public context, and $x^i \in \mathcal{X}_i$ is the private context belonging exclusively to player $i \in \{1, 2\}$. Upon the realization of the state, players choose actions $a^1 \in \mathcal{A}_1$ and $a^2 \in \mathcal{A}_2$. Let $\mathcal{A} \defeq \mathcal{A}_1 \times \mathcal{A}_2$ denote the joint action space. Player 1 receives a reward given by the function $r: \mathcal{S} \times \mathcal{A} \to [-1, 1]$. Because the game is zero-sum, Player 2 receives reward $-r(s, a)$.

\subsection{Information Constraints and Prescriptive Policies}

During execution, the latent state $s$ is partially observed. Player 1 observes $o^1 = (c, x^1)$ and Player 2 observes $o^2 = (c, x^2)$. Consequently, the players cannot execute standard Markov policies $\pi(a \mid s)$. Instead, they are restricted to \textit{prescriptive policies} that condition exclusively on their local observations. We define the feasible policy spaces $\Pi_1$ and $\Pi_2$ as the set of all such mappings:
$$\textstyle \Pi_i \defeq \left\{ \pi_i : \mathcal{C} \times \mathcal{X}_i \to \Delta(\mathcal{A}_i) \right\}, \quad \text{for } i \in \{1, 2\}.$$
Any deployed strategy must reside within these asymmetric information constraints.

\subsection{The Public Belief and Contextual Value}

Because neither player has full access to the latent state, strategic evaluation requires a common informational baseline. We define the \textbf{public belief} as the conditional probability distribution over the private types given the shared public context:
$$b(x^1, x^2 \mid c) \defeq \mathbb{P}_{s \sim \rho}\left(X^1=x^1, X^2=x^2 \mid C=c\right).$$
For any fixed public context $c \in \mathcal{C}$, the interaction reduces to a Bayesian game. If the players commit to prescriptive policies $(\pi_1, \pi_2) \in \Pi_1 \times \Pi_2$, the expected reward conditioned on the public context is evaluated by marginalizing over the public belief:
$$ r_c(\pi_1, \pi_2) \defeq \sum_{x^1 \in \mathcal{X}_1} \sum_{x^2 \in \mathcal{X}_2} b(x^1, x^2 \mid c) \sum_{a^1 \in \mathcal{A}_1} \sum_{a^2 \in \mathcal{A}_2} \pi_1(a^1 \mid c, x^1) \pi_2(a^2 \mid c, x^2) r(c, x^1, x^2, a^1, a^2).$$
The global expected value of the game for Player 1 is the expectation of these contextual values over the marginal distribution of the public context, $\rho_{\mathcal{C}}(c)$:
$$V(\pi_1, \pi_2) \defeq \mathbb{E}_{c \sim \rho_{\mathcal{C}}} \left[ r_c(\pi_1, \pi_2) \right].$$
\subsection{Offline Data and the CTDE Paradigm}

We operate under the paradigm of Centralized Training with Decentralized Execution (CTDE) \citep{lowe2017multi, foerster2018counterfactual}. We assume the offline learner (centralized trainer) has access to a \textit{full-state logged} dataset containing $n$ i.i.d.\ interactions generated by an unknown behavior policy pair $(\mu_1, \mu_2)$:
$$\mathcal{D} = \left\{ \left(c^{(k)}, x^{1(k)}, x^{2(k)}, a^{1(k)}, a^{2(k)}, r^{(k)}\right) \right\}_{k=1}^n.$$
This full-state logging assumption is standard in the multi-agent offline RL literature \citep{lyu2021contrasting, meng2023offline, pan2022plan, tseng2022offline, formanek2023off}. It also models platform-mediated markets, such as auctions, procurement, and logistics, where centralized historical logs are available for training while deployed agents act using only local information \citep{lowe2017multi}.

\subsection{Minimax Objective and the Exploitability Gap}
The canonical objective of the zero-sum game is the identification of the minimax value:
\begin{equation}
    V^\star \defeq \max_{\pi_1 \in \Pi_1} \min_{\pi_2 \in \Pi_2} V(\pi_1, \pi_2).
\end{equation}
The existence of this unique value and the well-posedness of the game under asymmetric information constraints are established in Section \ref{sec: existence of contextual value}.

Given the offline dataset $\mathcal{D}$, the learner's goal is to output a decentralized policy $\widehat{\pi}_1 \in \Pi_1$ for Player 1 that minimizes exploitability. We evaluate the performance of $\widehat{\pi}_1$ using the \textbf{exploitability gap}:
$$\textstyle Gap(\widehat{\pi}_1) \defeq V^\star - \min_{\pi_2 \in \Pi_2} V(\widehat{\pi}_1, \pi_2).$$
This measures the maximum exploitability of the learner's policy by an opponent equipped with perfect knowledge of $\widehat{\pi}_1$ but restricted to the feasible class $\Pi_2$.

\section{Information Asymmetry and Coverage Analysis}
\label{sec:coverage_analysis}

In this section, we introduce unilateral prescriptive concentrability and compare it with its fully observed counterpart. We show that asymmetric information changes equilibrium behavior and therefore the state-action regions that require coverage; depending on the logged behavior, the resulting coverage requirement can be smaller or larger.

\subsection{Existence of Contextual Value}
\label{sec: existence of contextual value}
 Because the game lacks temporal transitions, it decomposes into a collection of independent Bayesian games—one for each public context $c \in \mathcal{C}$. For a fixed public context $c$ and belief $b(\cdot, \cdot \mid c)$, the expected reward $r_c(\pi_{1}, \pi_{2})$ is bilinear in the prescriptive policies $\pi_i \in \Pi_i(c) \defeq \Delta(\Acal_i)^{\Xcal_i}$.

\begin{lemma}[Existence of Contextual Value \citep{kartik2021upper}]
For any public context $c \in \mathcal{C}$, there exists a unique value $V^\star(c)$ and a NE pair $(\pi_{1, \star}, \pi_{2, \star}) \in \Pi_1(c) \times \Pi_2(c)$ such that:
$$V^\star(c) = \max_{\pi_{1} \in \Pi_1(c)} \min_{\pi_{2} \in \Pi_2(c)} r_c(\pi_{1}, \pi_{2}) = \min_{\pi_{2} \in \Pi_2(c)} \max_{\pi_{1} \in \Pi_1(c)} r_c(\pi_{1}, \pi_{2}).$$
\end{lemma}

Since the public context is observed by both players, policies can be chosen independently across contexts, so $\Pi_i = \prod_{c \in \Ccal} \Pi_i(c)$ and $V^\star = \EE_{c \sim \rho_\Ccal}[V^\star(c)]$. We write $Gap_c(\pi_1) \defeq V^\star(c) - \min_{\pi_2 \in \Pi_2(c)} r_c(\pi_1, \pi_2)$ for the contextual exploitability gap, so that $Gap(\pi_1) = \EE_{c \sim \rho_\Ccal}[Gap_c(\pi_1)]$.

\vspace{-0.25em}
\subsection{Unilateral Prescriptive Coverage}
\vspace{-0.25em}
The feasibility of offline learning depends on the extent to which the logged data $\mathcal{D}$ covers the state-action regions relevant to the optimal strategy. In two-player games, a learner's policy $\widehat{\pi}_1$ must be robust against any possible adversarial response within the feasible policy class \citep{cui2022offline}. We formalize this via the \textit{unilateral prescriptive concentrability} coefficient.

\begin{assumption}[Unilateral Prescriptive Coverage]
\label{ass:unilateral_coverage}

Write $\Ecal_{\mathrm{Pres}} \defeq \mathrm{NE}(\Pi_1, \Pi_2)$ for the set of Nash equilibria of the prescriptive game and $\Ecal_{\mathrm{Full}} \defeq \mathrm{NE}(\Pi_1^{\mathrm{full}}, \Pi_2^{\mathrm{full}})$ for those of the fully observed game. We define the \emph{unilateral prescriptive concentrability} $C^\star_{\mathrm{Pres}}$ by minimising over $\pi_\star = (\pi_{1,\star}, \pi_{2,\star}) \in \Ecal_{\mathrm{Pres}}$:

\begin{align}
C^\star_{\text{Pres}} \defeq \min_{\pi_\star \in \Ecal_{\mathrm{Pres}}} \max \left\{ \sup_{\nu \in \Pi_2}\max_{s \in \Scal, a \in \Acal} \frac{d^{\pi_{1, \star}, \nu}(s, a)}{d^\rho(s, a)}, \sup_{\mu \in \Pi_1}\max_{s \in \Scal, a \in \Acal} \frac{d^{\mu, \pi_{2, \star}}(s, a)}{d^\rho(s, a)} \right\}, \label{eq:prescriptive_c}
\end{align}

where $d^{\pi_1, \pi_2}(s, a) = \rho(c, x^1, x^2) \pi_1(a^1 \mid c, x^1) \pi_2(a^2 \mid c, x^2)$ is the density induced by information-constrained strategies, and $d^\rho(s, a)$ is the latent state-action density in the logged dataset $\Dcal$. We assume $C^\star_{\text{Pres}} < \infty$.
\end{assumption}

We denote $d^\rho_{\min} = \min_{(s,a) \in \text{supp}(d^\rho)} d^\rho(s, a)$ as the minimum density of the supported state-action pairs in the behavior distribution.
\vspace{-0.25em}
\subsection{Information Structure and Offline Coverage}
\vspace{-0.25em}
A central question in multi-agent offline RL is how information constraints affect the ``data-hardness'' of the game. To quantify this effect, we compare our prescriptive concentrability to the standard \textit{full-information concentrability} used in fully observed interactions \citep{cui2022offline}.

In a fully observed baseline, players are not restricted by private information and can choose any policy $\pi_i \in \Pi_i^{\text{full}}$ that conditions on the entire latent state $s \in \mathcal{S}$. Following \citet{cui2022offline}, the coefficient is defined as:

\begin{equation}
\label{eq:standard_c}
C^\star_{\text{Full}} \defeq \min_{\pi_\star \in \Ecal_{\mathrm{Full}}} \max \left\{ \sup_{\nu \in \Pi_2^{\text{full}}}\max_{s, a} \frac{d^{\pi_{1, \star}, \nu}(s, a)}{d^\rho(s, a)}, \sup_{\mu \in \Pi_1^{\text{full}}}\max_{s, a} \frac{d^{\mu, \pi_{2, \star}}(s, a)}{d^\rho(s, a)} \right\}.
\end{equation}
In $C^\star_{\text{Pres}}$, policies condition only on local observations, whereas in $C^\star_{\text{Full}}$ they condition on the full state. The two coefficients are also evaluated at equilibria of their respective games. Their difference can therefore arise from two sources: the equilibrium induced by the information structure and the information available to a unilateral deviator.

We first establish that the equilibrium effect can produce a coverage gap that scales with the number of private types.

\begin{proposition}[Equilibrium-Induced Coverage Separation]
\label{prop:separation}
For every integer $K \ge 3$ there is a public--private zero-sum game with $K$ private types per player and a full-support behavior policy for which $C^\star_{\text{Pres}} = 4K$ and $C^\star_{\text{Full}} = 4K^2$.
\end{proposition}
Its proof is in Appendix~\ref{proof:separation}, which also shows that under other logging policies the comparison can reverse. The proposition shows that asymmetric information can change the coverage that offline learning requires through its effect on equilibrium behavior. We next ask whether the larger deviation class available under full information creates an additional difference. To isolate this effect, we hold the same reference policy fixed in both comparisons.

\begin{lemma}[Equality at a Common Reference Policy]
\label{lem:fixed_anchor}
For any reference policy $\bar{\pi}_1 : \Scal \to \Delta(\Acal_1)$ and any logged density $d^\rho$,
\begin{equation*}
\sup_{\nu \in \Pi_2} \max_{s, a} \frac{d^{\bar{\pi}_1, \nu}(s, a)}{d^\rho(s, a)}
= \sup_{\nu \in \Pi_2^{\text{full}}} \max_{s, a} \frac{d^{\bar{\pi}_1, \nu}(s, a)}{d^\rho(s, a)},
\end{equation*}
where $d^{\bar{\pi}_1, \nu}(s, a) = \rho(s)\, \bar{\pi}_1(a^1 \mid s)\, \nu(a^2 \mid \cdot)$ with each policy conditioned on what it observes; symmetrically for Player~1's deviations against a fixed $\bar{\pi}_2$.
\end{lemma}
Its proof is in Appendix~\ref{proof:fixed_anchor}. Because the criterion maximizes over a single state-action tuple, a prescriptive deviator can place all of its mass on the maximizing action at the corresponding local observation without observing the rival's private type. We also show in Appendix~\ref{app: proof relaxation} that the same equality holds for an integrated coverage criterion under decentralized logging.

Together, these results show that information structure affects offline coverage through the equilibrium behavior it induces. The comparison with fully observed games depends on how the logged data cover the respective equilibria, while at a fixed reference policy the two deviation classes yield the same pointwise coefficient. This motivates the unilateral prescriptive coverage condition used in the learning results below.

\vspace{-0.25em}
\section{Offline Learning for Finite Asymmetric Games}
\label{sec:finite_learning}
\vspace{-0.25em}
In this section, we develop a pessimistic algorithm for two-player zero-sum public--private contextual games and prove an $\tilde{\mathcal{O}}(1/\sqrt{n})$ exploitability rate, matching the standard sample-size dependence for fully observed minimax games.

\subsection{Empirical Estimators and Pessimistic Objective}

The learner utilizes the full-state logged dataset $\mathcal{D} = \{ (s^{(k)}, a^{(k)}, r^{(k)}) \}_{k=1}^n$ to estimate the game parameters. For each state-action pair $(s, a) \in \mathcal{S} \times \mathcal{A}$, we define the visit count $N(s, a) \defeq \sum_{k=1}^n \mathbb{I} \{ (s, a)^{(k)} = (s, a) \}$ and the corresponding empirical reward:
\begin{equation}
 \textstyle \wh{r}(s, a) \defeq \frac{\sum_{k=1}^n r^{(k)} \mathbb{I} \{ (s, a)^{(k)} = (s, a) \}}{N(s, a) \vee 1}.
\end{equation}
We further define the marginal state counts $N(s) \defeq \sum_{a \in \mathcal{A}} N(s, a)$ and the public contextual counts $N(c) \defeq \sum_{x^1 \in \mathcal{X}_1, x^2 \in \mathcal{X}_2} N(c, x^1, x^2)$. To protect against distribution shift, we construct the \textit{clipped pessimistic reward} $\underline{r}(s, a)$ using an uncertainty bonus $\beta(s, a)$:
\begin{equation}
\vspace{-0.25em}
 \underline{r}(s, a) = \max \{ \wh{r}(s, a) - \beta(s, a), -1 \}, \quad \text{where } \beta(s, a) = \sqrt{2} \sqrt{\frac{\iota}{N(s, a) \vee 1}}.
\end{equation}
Here, $\iota \defeq \log(4|\mathcal{S}||\mathcal{A}|/\delta)$ is the concentration parameter. For each public context $c$, let

\begin{equation}
 \textstyle \wh{b}(x^1, x^2 \mid c) \defeq
 \begin{cases}
 N(c, x^1, x^2) / N(c), & N(c) > 0, \\[2pt]
 b_0(x^1, x^2 \mid c), & N(c) = 0,
 \end{cases}
\end{equation}
where $b_0(\cdot \mid c)$ is any fixed distribution on $\Xcal_1 \times \Xcal_2$. We use the conditionals $\wh{b}(x^{-i} \mid c, x^i)$ induced by $\wh{b}$, with arbitrary values on zero-probability conditioning events.
We define the \textbf{pessimistic contextual objective} $L_c(\pi_1, \pi_2)$ as the expected value of the conservative reward under the empirical belief:
\begin{equation}
 \textstyle L_c(\pi_1, \pi_2) \defeq \sum_{x^1, x^2} \wh{b}(x^1, x^2 \mid c) \sum_{a^1, a^2} \pi_1(a^1 \mid c, x^1) \pi_2(a^2 \mid c, x^2) \underline{r}(c, x^1, x^2, a^1, a^2).
\end{equation}
To address the challenge of learning from asymmetric offline data, we propose the Public-Private Asymmetric Exact Pessimism (PPA-Exact) algorithm. The algorithm proceeds in two main stages: first, it leverages the ``global oracle'' properties of the logged dataset to perform joint statistical estimation of the rewards and the public belief \citep{kartik2021upper}. Second, it decomposes the global interaction into a collection of independent Bayesian games, identifying the minimax equilibrium over the space of prescriptive policies for each public context. By constructing conservative value estimates, the learner ensures that the resulting decentralized policy remains robust against adversarial exploitation in regions with sparse data coverage.

\begin{algorithm}[H]
\caption{Public-Private Asymmetric Exact Pessimism (PPA-Exact)}
\label{alg:ppa_exact}
\begin{algorithmic}[1]
\Require Offline dataset $\mathcal{D} = \{ (c^{(k)}, x^{1,(k)}, x^{2,(k)}, a^{1,(k)}, a^{2,(k)}, r^{(k)}) \}_{k=1}^n$, confidence parameter $\delta \in (0,1)$.
\State \textbf{Initialize:} For all $s \in \mathcal{S}, a \in \mathcal{A}$, calculate visit counts $N(s,a) = \sum_{k=1}^n \mathbb{I}\{ (s, a)^{(k)} = (s, a) \}$.
\State \textbf{Estimate Beliefs:} For each public context $c \in \mathcal{C}$, compute the empirical public belief:

\[
\wh{b}(x^1, x^2 \mid c) = \frac{\sum_{a^1, a^2} N(c, x^1, x^2, a^1, a^2)}{N(c)} \text{ if } N(c) > 0, \quad \text{and } b_0(\cdot \mid c) \text{ otherwise.}
\]

\State \textbf{Construct Pessimistic Rewards:} Let $\iota = \log(4|\mathcal{S}||\mathcal{A}|/\delta)$. For all $s \in \mathcal{S}, a \in \mathcal{A}$:
\State \quad Compute empirical reward: $\wh{r}(s, a) = \frac{\sum_{k=1}^n r^{(k)} \mathbb{I}\{ (s, a)^{(k)} = (s, a) \}}{N(s, a) \vee 1}$
\State \quad Compute uncertainty bonus: $\beta(s, a) = \sqrt{2} \sqrt{\frac{\iota}{N(s, a) \vee 1}}$
\State \quad Set clipped pessimistic reward: $\underline{r}(s, a) = \max \{ \wh{r}(s, a) - \beta(s, a), -1 \}$
\For {each public context $c \in \mathcal{C}$}
 \State Define the expected pessimistic objective over prescriptions $\pi_{1} \in \Pi_1(c)$ and $\pi_{2} \in \Pi_2(c)$:
 \[
 L_c(\pi_1, \pi_2) = \sum_{x^{1:2}, a^{1:2}} \wh{b}(x^1, x^2 \mid c) \pi_{1}(a^1 \mid c, x^1) \pi_{2}(a^2 \mid c, x^2) \underline{r}(s, a)
 \]
 \State Solve the resulting bilinear minimax optimization via Linear Programming:
 \[
 \wh{\pi}_{1,c} \in \argmax_{\pi_1 \in \Pi_1(c)} \min_{\pi_2 \in \Pi_2(c)} L_c(\pi_1, \pi_2)
 \]
\EndFor
\Ensure The information-constrained policy for player 1: $\wh{\pi}_1 = \{ \wh{\pi}_{1,c} \}_{c \in \mathcal{C}}$.
\end{algorithmic}
\end{algorithm}
\subsection{Sample Complexity Analysis}

Our main result characterizes the performance of PPA-Exact in terms of the exploitability gap.
\begin{theorem}
\label{thm:finite_rate}
Suppose the dataset size satisfies $n \ge \frac{8}{d^\rho_{\min}} \log\left( \frac{2 |\mathcal{S}| |\mathcal{A}|}{\delta} \right)$. Under Assumption~\ref{ass:unilateral_coverage}, it holds with probability at least $1-\delta$ that the policy $\widehat{\pi}_1$ obtained from Algorithm~\ref{alg:ppa_exact} satisfies:
\begin{equation}
 \textstyle Gap(\widehat{\pi}_1) \le 2(2 + 2\kappa_b) \sqrt{\frac{C^\star_{\text{Pres}} |\mathcal{S}| |\mathcal{A}| \iota}{n}},
\end{equation}
where $\kappa_b \defeq \frac{\sqrt{|\mathcal{X}_1| |\mathcal{X}_2|} + \sqrt{2}}{\sqrt{2}}$ and $\iota \defeq \log(4|\mathcal{S}||\mathcal{A}|/\delta)$. Its full proof is in Appendix \ref{proof:finite_rate}.
\end{theorem}

\paragraph{Proof Sketch.} The proof proceeds in three stages.
First, we establish \textit{approximate pessimism}: since $\underline{r} \le r$ (Lemma~\ref{lem:det_pessimism}), $L_c(\pi_1, \pi_2) \le r_c(\pi_1, \pi_2) + \|b - \wh{b}\|_1$, so $L_c$ lower-bounds the true value up to the belief-estimation error. We also bound the underestimation, $r_c(\pi_1, \pi_2) - L_c(\pi_1, \pi_2) \le \mathbb{E}_{s, a}[2\beta] + \|b - \wh{b}\|_1$ \eqref{eq: approximate pessimism decomposition}. 

Second, we utilize the minimax optimality of $\wh{\pi}_1$ to bound the contextual gap by the pessimism of the equilibrium policy: $Gap_c(\wh{\pi}_1) \le \mathbb{E}_{d^{\pi_{1, \star}, \wh{\pi}_2} \mid c} [ (2 + 2\kappa_b)\beta(s, a) ]$ \eqref{eq:contextual_gap_bound}, where $\wh{\pi}_2$ is the empirical best response.

Third, we integrate this bound over the public contexts \eqref{eq: global expectation}. By applying a count concentration event and the Cauchy-Schwarz inequality, we bound the sum of uncertainty bonuses by the unilateral prescriptive concentrability $C^\star_{\text{Pres}}$ \eqref{eq: global gap through CS}. This yields the stated $\tilde{\mathcal{O}}(1/\sqrt{n})$ sample-size dependence.

While \texttt{PPA-Exact} (Algorithm~\ref{alg:ppa_exact}) leverages Linear Programming to identify the minimax equilibrium exactly, this centralized approach faces a “memory wall” in high-dimensional environments where the private information space $|\mathcal{X}|$ is massive. To address this, we provide a decentralized iterative alternative, \texttt{PPA-Learning} (Algorithm~\ref{alg:mwu_solver}), in Appendix~\ref{sec:practical_optimization}. By decomposing the global minimax objective into local decision problems through no-regret dynamics, we establish a unified exploitability gap of $\tilde{\mathcal{O}}(1/\sqrt{n} + 1/\sqrt{T})$, characterizing the joint impact of statistical estimation and optimization regret. This iterative framework serves as the theoretical foundation for our extension to general function approximation in Section~\ref{sec:function_approximation}.

\vspace{-0.25em}
\section{General Function Approximation}
\label{sec:function_approximation}
\vspace{-0.25em}
The previous analysis of finite games establishes the fundamental sample complexity of the public--private setting; however, the polynomial dependence on state-action cardinality limits applicability to large-scale environments. To address this, we extend the framework to general function approximation by replacing counting-based estimators with supervised learning and utilizing Policy Mirror Descent for decentralized optimization.
\vspace{-0.25em}
\subsection{Offline Estimation and Version Space Pessimism}
\vspace{-0.25em}
Instead of point-wise estimation, we utilize function generalization across the latent state $s = (c, x^1, x^2)$ and joint actions $a = (a^1, a^2)$.
\vspace{-0.25em}
\paragraph{Belief and Reward Estimation.} We approximate the public belief using a conditional density class $\mathcal{B}$. The learner identifies the empirical best belief $\wh{b} \in \mathcal{B}$ via Maximum Likelihood Estimation (MLE):
\begin{equation}
   \textstyle \wh{b} = \argmax_{b \in \mathcal{B}} \frac{1}{n} \sum_{k=1}^n \log b(x^{1, (k)}, x^{2, (k)} \mid c^{(k)}). \label{eq:belief_mle}
\end{equation}
Similarly, we model the joint reward using a regressor class $\mathcal{F}$ and compute the Empirical Risk Minimizer (ERM) $\wh{r} \in \mathcal{F}$ by minimizing the mean squared error over $\mathcal{D}$: $ \textstyle \mathcal{L}_{\mathcal{D}}(f) \defeq \frac{1}{n} \sum_{k=1}^n \left( f(s^{(k)}, a^{(k)}) - r^{(k)} \right)^2,$ and $\wh{r} = \argmin_{f \in \mathcal{F}} \mathcal{L}_{\mathcal{D}}(f)$.
To quantify statistical uncertainty in the FA regime, we adopt the principle of version space pessimism \citep{xie2021bellman}. Because our setting is a contextual bandit ($H=1$), the Bellman error collapses to the standard regression error. We construct the version space $\mathcal{F}(\beta) \subseteq \mathcal{F}$ as the set of statistically plausible reward functions with near-optimal empirical risk:
\begin{equation}
  \textstyle \mathcal{F}(\beta) \defeq \left\{ f \in \mathcal{F} : \mathcal{L}_{\mathcal{D}}(f) - \mathcal{L}_{\mathcal{D}}(\wh{r}) \le \beta \right\},
\end{equation}
where $\beta > 0$ is a tolerance parameter scaled by the statistical complexity of $\mathcal{F}$. To protect against adversarial exploitation of poorly covered regions, the trainer constructs the globally pessimistic reward predictor by taking the lower envelope of this space:
\begin{equation}
  \textstyle \underline{r}(s, a) \defeq \inf_{f \in \mathcal{F}(\beta)} f(s, a). \label{eq:version_space_pessimism}
\end{equation}

\subsection{Asymmetric Policy Mirror Descent}

With the pessimistic critic $\underline{r}$ and belief $\wh{b}$ established, the decentralized actors $\pi_{\theta_1}$ and $\pi_{\theta_2}$ update their parameters via iterative mirror descent. At each iteration $t$, the algorithm evaluates the conditionally expected local $Q$-values for both players against the opponent's current strategy:
\begin{align}
    Q_{1,t}(c, x^1, a^1) &= \mathbb{E}_{x^2 \sim \wh{b}(\cdot \mid c, x^1), a^2 \sim \pi_{\theta_{2,t}}} \left[ \underline{r}(c, x^1, x^2, a^1, a^2) \right], \label{eq:fa_q1} \\
    Q_{2,t}(c, x^2, a^2) &= \mathbb{E}_{x^1 \sim \wh{b}(\cdot \mid c, x^2), a^1 \sim \pi_{\theta_{1,t}}} \left[ \underline{r}(c, x^1, x^2, a^1, a^2) \right]. \label{eq:fa_q2}
\end{align}
Here $\wh{b}(x^{-i} \mid c, x^i)$ denotes the conditional induced by $\wh{b}$. The actors then take mirror-descent steps against these $Q$-values. Any update satisfying the external-regret condition of Assumption~\ref{ass:actor_regret} is admissible; Appendix~\ref{app:fa_discussion} gives one parameterised implementation (Algorithm~\ref{alg:paac}).

\subsection{Theoretical Analysis}

To establish the sample complexity of PPA-PMD, we adopt the standard assumption that the chosen function classes are sufficiently expressive to contain the ground-truth distributions \citep{xie2021batch, zhang2023offline}.

\begin{assumption}[Realizability]
\label{ass:realizability}
The true joint reward function is contained in the regressor class, $r \in \mathcal{F}$. The true public belief distribution is contained in the density class, $b(\cdot \mid c) \in \mathcal{B}$ for all $c \in \mathcal{C}$. Every candidate is uniformly bounded on the same scale as the reward: $\mathcal{F} \subseteq [-1, 1]^{\mathcal{S} \times \mathcal{A}}$.
\end{assumption}

To handle general function approximation, we require the dataset to capture the structural variance of the version space. Let $d^{\pi_1, \pi_2}$ be the joint state-action distribution. For the equilibrium $\pi_{1, \star}$ and its empirical best-response $\wh{\pi}_2 = \argmin_{\pi_2} L_c(\pi_{1, \star}, \pi_2)$, we have:

\begin{assumption}[Prescriptive Uniform FA Concentrability]
\label{ass:fa_concentrability}
There exists a finite constant $\Cscr_{\Fcal} < \infty$ such that, for any fixed unilateral prescriptive deviation $\pi' \in \Pi_{\mathrm{pres\text{-}uni}}$ and \emph{every} subset $\Gcal \subseteq \Fcal$ with $r \in \Gcal$ (in particular $\Gcal = \Fcal(\beta)$),

\begin{equation}
 \EE_{s, a \sim d^{\pi'}} \left[ \sup_{f \in \Gcal} (f(s, a) - r(s, a))^2 \right] \le \Cscr_{\Fcal} \cdot \sup_{f \in \Gcal} \EE_{s, a \sim d^\rho} \left[ (f(s, a) - r(s, a))^2 \right],
\end{equation}

where $\Pi_{\mathrm{pres\text{-}uni}}$ denotes reference-anchored unilateral prescriptive deviations from the equilibrium $\pi_{1,\star}$, including the unilateral deviation pair $(\pi_{1,\star}, \wh{\pi}_2)$.
\end{assumption}

Assumption~\ref{ass:fa_concentrability} generalizes prescriptive unilateral coverage (Assumption~\ref{ass:unilateral_coverage}) to function space, following the convention used in offline games \citep{zhang2023offline}.
This formulation is analogous to the integral concentrability coefficients used to handle version-space diameters in offline RL \citep{uehara2021pessimistic, yin2022near} and enables the pointwise version-space pessimism in \eqref{eq:version_space_pessimism}. Because it bounds the version-space envelope rather than each regressor's error, $\Cscr_\Fcal$ can also reflect the richness of $\Fcal$. 

We assume the actor updates satisfy a standard external-regret guarantee with respect to the pessimistic contextual objective.

\begin{assumption}[Actor No-Regret Guarantee]
\label{ass:actor_regret}
The iterates $(\pi_{1,t}, \pi_{2,t})_{t=1}^T$ produced by the actor updates admit external regret bounds $R_1(T)$ and $R_2(T)$ that hold uniformly over public contexts: for every $c \in \Ccal$,
\begin{align}
\sup_{\pi_1 \in \Pi_1} \sum_{t=1}^T \left[ L_c(\pi_1, \pi_{2,t}) - L_c(\pi_{1,t}, \pi_{2,t}) \right] &\le R_1(T), \\
\sup_{\pi_2 \in \Pi_2} \sum_{t=1}^T \left[ L_c(\pi_{1,t}, \pi_{2,t}) - L_c(\pi_{1,t}, \pi_2) \right] &\le R_2(T).
\end{align}
\end{assumption}

Assumption~\ref{ass:actor_regret} is a standard external-regret condition \citep{freund1999adaptive, cesa2006prediction}. In particular, multiplicative-weights updates satisfy $R_i(T) \le 2\sqrt{T \log(\max_i |\Acal_i|)}$, as shown in Appendix~\ref{sec:practical_optimization}.

We now state the performance guarantee.

\begin{theorem}[Function-Approximation Exploitability Gap]
\label{thm:paac_gap}
Suppose the reward regressor class $\mathcal{F}$ and the belief density class $\mathcal{B}$ are finite. Under Assumptions \ref{ass:realizability}, \ref{ass:fa_concentrability}, and \ref{ass:actor_regret}, if we set the version space tolerance to $\beta = \frac{136 \log(2|\mathcal{F}|/\delta)}{n}$, then with probability at least $1-\delta$, the exploitability gap of the policy $\bar{\pi}_1$ produced by these updates satisfies:

\begin{equation}
 Gap(\bar{\pi}_1) \le 2 \sqrt{\frac{272 \Cscr_{\Fcal} \log(2|\mathcal{F}|/\delta)}{n}} + 4 \sqrt{\frac{2 \log(2|\mathcal{B}|/\delta)}{n}} + \frac{R_1(T) + R_2(T)}{T}.
\end{equation}

\end{theorem}

For multiplicative-weights actor updates, $R_i(T) \le 2\sqrt{T \log(\max_i |\mathcal{A}_i|)}$, so the last term is $O(\sqrt{\log(\max_i |\mathcal{A}_i|)/T})$ and the bound becomes $\tilde{\mathcal{O}}(1/\sqrt{n} + 1/\sqrt{T})$. The parameterised minibatch implementation in Algorithm~\ref{alg:paac} is a practical variant; its stochastic optimisation error is not analysed here (Appendix~\ref{app:fa_discussion}).

\paragraph{Proof Sketch.} Its full proof is in Appendix~\ref{proof:paac_gap}. The proof decomposes the gap into value, optimization, and pessimism errors. The pessimism error is bounded by the diameter of $\Fcal(\beta)$ and the $L_1$ distance of the empirical belief $\wh{b}$. That $L_1$ error is controlled by a finite-class Hellinger likelihood-ratio bound. We then bound the version-space squared error on the behavior measure via a fast-rate uniform Bernstein inequality, which requires setting $\beta \propto \log(|\mathcal{F}|)/n$ to apply a union bound over the finite class $\Fcal$. We transfer this statistical error to the target trajectory by applying Cauchy--Schwarz directly to the random variable $\sup_{f\in\Fcal(\beta)}(f-r)^2$, followed by the uniform-ratio assumption $\Cscr_{\Fcal}$ in a single step. Substituting the actor regret of Assumption~\ref{ass:actor_regret} gives the stated $\tilde{\mathcal{O}}(1/\sqrt{n} + 1/\sqrt{T})$ rate. \vspace{-0.25em}
\section{Related Work}
\label{sec:related_work}
\vspace{-0.25em}
\paragraph{The Principle of Pessimism in Offline Learning.}
Offline reinforcement learning (RL) is fundamentally challenged by distribution shift, where the learner must infer an optimal policy from a fixed behavior dataset that may lack coverage of critical state--action regions \citep{levine2020offline}. To provide statistical guarantees in the face of such uncertainty, the principle of \textit{pessimism} has become the standard theoretical tool \citep{jin2021pessimism, rashidinejad2021bridging}. By constructing lower confidence bounds (LCB) or conservative value estimates, pessimistic algorithms prevent the learner from being misled by overestimation bias in out-of-distribution regions \citep{xie2021bellman, uehara2021pessimistic}. While this framework has been shown to be minimax optimal for single-agent environments under various concentrability assumptions \citep{jin2021pessimism, rashidinejad2021bridging}, these results typically assume that the learner has access to the full state of the environment.
\vspace{-0.25em}
\paragraph{Offline Learning for Two-Player Zero-Sum Games.}
Learning minimax equilibria in competitive environments extends the challenges of offline RL to strategic interactions \citep{littman1994markov}. In two-player zero-sum games, the identification of a Nash Equilibrium (NE) requires the dataset to cover not only the equilibrium path but also any possible unilateral deviations by an adversary \citep{cui2022offline}. Recent work has formalized this via the \textit{unilateral concentrability} coefficient, proving it to be a necessary and sufficient condition for solvability in offline Markov games \citep{cui2022provably, zhong2022pessimistic}. These approaches typically extend pessimism to the game-theoretic setting by constructing conservative value iterations for the minimax objective \citep{zhong2022pessimistic, zhang2023offline}. Complementary recent work studies KL-regularized offline games, including pessimism-free approaches in two-player zero-sum, general-sum, and potential-game settings \citep{zhang2026pessimism,chen2026offline,chen2026fast,chen2026pessimism}. However, existing literature largely assumes a fully observed state space where all players possess perfect information~\citep{cui2022offline, zhong2022pessimistic, zhang2023offline}. Our analysis provides, to our knowledge, the first theoretical analysis of offline equilibrium learning in zero-sum games with asymmetric information. We introduce \textit{unilateral prescriptive concentrability} and show that information structure changes offline coverage through the equilibrium it induces.
\vspace{-0.25em}
\paragraph{Information Asymmetry and Bayesian Strategic Games.}
Strategic interactions involving hidden information are formally characterized as Bayesian games or games with incomplete information \citep{harsanyi1967games}. Such models are fundamental to Operations Research and Economics, underpinning the design of complex auction mechanisms \citep{milgrom1982theory, ma2022revenue} and automated negotiation protocols \citep{roughgarden2010algorithmic,  liu2026learning, liu2026strategic, liu2026instructing, chen2026learning} where agents possess private signals hidden from their rivals. In modern multi-agent RL, this asymmetric structure is commonly navigated via the Centralized Training with Decentralized Execution (CTDE) paradigm \citep{lowe2017multi}. While the existence of game values and the characterization of optimal strategies in these environments have been established from a planning perspective \citep{kartik2021upper}, the problem of recovering these decentralized strategies from fixed, full-state historical logs remains largely unexplored. Our work bridges the gap between these game-theoretic foundations and offline RL by providing the first statistical framework that leverages the public belief to identify minimax-optimal policies directly from logged data.
\vspace{-0.25em}
\section{Conclusion}
\label{sec:discussion}
\vspace{-0.25em}
In this paper, we studied offline learning in two-player zero-sum contextual games with public and private information, a structure common in economic interactions such as auctions, procurement, and negotiation. We introduced \textit{unilateral prescriptive concentrability}, a coverage notion defined over the prescriptive policy space, and developed pessimistic learning methods that account for both statistical uncertainty and the information constraints of decentralized execution. For finite games, \texttt{PPA-Exact} achieves an $\tilde{\mathcal{O}}(1/\sqrt{n})$ exploitability rate, matching the standard sample-size dependence for fully observed minimax games, while our policy mirror descent framework for general function approximation, \texttt{PPA-PMD}, achieves $\tilde{\mathcal{O}}(1/\sqrt{n} + 1/\sqrt{T})$ with no-regret actor updates. Extending these results to multi-stage games with dynamic beliefs and strategic signaling is a natural direction for future work.

\section*{AI Use Statement}
\par
Generative AI tools were used solely for grammar correction. They were not used to conduct experiments, generate results, or for research ideation, methodology, theoretical development or proofs, code, substantive data analysis, or interpretation of results. All AI-assisted work was reviewed by the authors, who take responsibility for the final content.
\par

\bibliographystyle{iclr2027_conference}
\bibliography{bibliography}

@article{ma2022revenue,
  title={Revenue-optimal deterministic auctions for multiple buyers with ordinal preferences over fixed-price items},
  author={Ma, Will},
  journal={ACM Transactions on Economics and Computation},
  volume={10},
  number={2},
  pages={1--32},
  year={2022},
  publisher={ACM New York, NY}
}

@article{jin2024sample,
  title={Sample complexity of posted pricing for a single item},
  author={Jin, Billy and Kesselheim, Thomas and Ma, Will and Singla, Sahil},
  journal={Advances in Neural Information Processing Systems},
  volume={37},
  pages={82296--82317},
  year={2024}
}

@misc{chen2026learning,
  title  = {Learning to Recommend and Negotiate in Multi-Item Bargaining via Curriculum Reinforcement Learning},
  author = {Chen, Claire and Liu, Shuze Daniel and Tan, Renfei and Simchi-Levi, David and Joachims, Thorsten},
  year   = {2026},
  note   = {Manuscript}
}

@article{liu2026learning,
  title={Learning to Sell: Reinforcement Learning for Strategic Large Language Model Agents in Multi-Product Markets},
  author={Liu, Shuze Daniel and Chen, Claire and Wang, Jiuqi and Simchi-Levi, David and Joachims, Thorsten},
  journal={arXiv preprint arXiv:2609.33289},
  year={2026}
}

@phdthesis{liu2025efficientrobust,
  title     = {Efficient and Robust Policy Evaluation for Reinforcement Learning},
  author    = {Liu, Shuze Daniel},
  school    = {University of Virginia},
  year      = {2025}
}

@article{xie2026beyond,
  title={Beyond Linear Attention: {Softmax} Transformers Implement In-Context Reinforcement Learning},
  author={Xie, Zixuan and Liu, Xinyu and Chen, Claire and Liu, Shuze Daniel and Chandra, Rohan and Zhang, Shangtong},
  journal={arXiv preprint arXiv:2605.07333},
  year={2026}
}

@article{zhang2026pessimism,
      title={Beyond Pessimism: Offline Learning in KL-regularized Games}, 
      author={Yuheng Zhang and Claire Chen and Nan Jiang},
        journal={Advances in Neural Information Processing Systems},
  year={2026}
}

@inproceedings{chen2026offline,
title={Offline Two-Player Zero-Sum Markov Games with {KL} Regularization},
author={Claire Chen and Yuheng Zhang and Xinyu Liu and Zixuan Xie and Shuze Daniel Liu and Nan Jiang},
booktitle={Forty-third International Conference on Machine Learning},
year={2026},
url={https://openreview.net/forum?id=cQD2wxFhFG}
}

@article{chen2026pessimism,
  title={Pessimism-Free Offline Learning in General-Sum Games via {KL} Regularization},
  author={Chen, Claire and Zhang, Yuheng},
  journal={ArXiv Preprint arXiv:2605.00264},
  year={2026}
}

@article{chen2026fast,
  title={Fast Rates in $\alpha$-Potential Games via Regularized Mirror Descent},
  author={Chen, Claire and Zhang, Yuheng},
  journal={ArXiv Preprint arXiv:2605.00268},
  year={2026}
}

@article{liu2026strategic,
  title={Strategic Bargaining in Multi-Buyer Markets: Reinforcement Learning from Verifiable Rewards for {LLM} Negotiations},
  author={Liu, Shuze Daniel and Chen, Claire and Xiao, Jiabao Sean and Chen, Xin and Simchi-Levi, David},
  journal={arXiv preprint arXiv:2607.05863},
  year={2026}
}

@inproceedings{mahadevan2026convergence,
  title={Convergence of Two-Timescale {Markovian} Stochastic Approximations with Applications in Reinforcement Learning},
  author={Mahadevan, Vagul and Chen, Claire and Liu, Shuze Daniel and Zhang, Shangtong},
  booktitle={Proceedings of the 43rd International Conference on Machine Learning},
  year={2026}
}

@article{liu2026instructing,
  title={Instructing {LLM}s to Negotiate Using Reinforcement Learning with Verifiable Rewards},
  author={Liu, Shuze Daniel and Chen, Claire and Xiao, Jiabao Sean and Lei, Lei and Zhang, Yuheng and Yue, Yisong and Simchi-Levi, David},
  journal={Advances in Neural Information Processing Systems},
  year={2026}
}

@article{liu2026or,
  title={OR-Transformer: Scaling Real-Time Decision-Making to 1,000 Items},
  author={Liu, Shuze Daniel and Simchi-Levi, David and Chen, Claire and Gao, Chutong and Zhang, Shangtong},
  journal={arXiv preprint arXiv:2609.01933},
  year={2026}
}

@article{chen2026robust,
  title={Robust Data-Collection Policy Learning for Low-Variance Online Policy Evaluation},
  author={Chen, Claire and Liu, Shuze Daniel and Luo, Licheng and Chandra, Rohan and Jiang, Nan and Zhang, Shangtong},
  journal={arXiv preprint arXiv:2608.24146},
  year={2026}
}

@inproceedings{pan2022plan,
  title={Plan better amid conservatism: Offline multi-agent reinforcement learning with actor rectification},
  author={Pan, Ling and Huang, Longbo and Ma, Tengyu and Xu, Huazhe},
  booktitle={International conference on machine learning},
  pages={17221--17237},
  year={2022},
  organization={PMLR}
}

@article{tseng2022offline,
  title={Offline multi-agent reinforcement learning with knowledge distillation},
  author={Tseng, Wei-Cheng and Wang, Tsun-Hsuan Johnson and Lin, Yen-Chen and Isola, Phillip},
  journal={Advances in Neural Information Processing Systems},
  volume={35},
  pages={226--237},
  year={2022}
}

@article{formanek2023off,
  title={Off-the-grid marl: Datasets with baselines for offline multi-agent reinforcement learning},
  author={Formanek, Claude and Jeewa, Asad and Shock, Jonathan and Pretorius, Arnu},
  journal={arXiv preprint arXiv:2302.00521},
  year={2023}
}

@article{swaminathan2015batch,
  title={Batch learning from logged bandit feedback through counterfactual risk minimization},
  author={Swaminathan, Adith and Joachims, Thorsten},
  journal={The Journal of Machine Learning Research},
  volume={16},
  number={1},
  pages={1731--1755},
  year={2015},
  publisher={JMLR. org}
}

@article{yin2022near,
  title={Near-optimal offline reinforcement learning with linear representation: Leveraging variance information with pessimism},
  author={Yin, Ming and Duan, Yaqi and Wang, Mengdi and Wang, Yu-Xiang},
  journal={arXiv preprint arXiv:2203.05804},
  year={2022}
}

@book{cesa2006prediction,
  title={Prediction, learning, and games},
  author={Cesa-Bianchi, Nicolo and Lugosi, G{\'a}bor},
  year={2006},
  publisher={Cambridge university press}
}

@article{arora2012multiplicative,
  title={The multiplicative weights update method: a meta-algorithm and applications},
  author={Arora, Sanjeev and Hazan, Elad and Kale, Satyen},
  journal={Theory of computing},
  volume={8},
  number={1},
  pages={121--164},
  year={2012},
  publisher={Theory of Computing Exchange}
}

@article{meng2023offline,
  title={Offline pre-trained multi-agent decision transformer},
  author={Meng, Linghui and Wen, Muning and Le, Chenyang and Li, Xiyun and Xing, Dengpeng and Zhang, Weinan and Wen, Ying and Zhang, Haifeng and Wang, Jun and Yang, Yaodong and others},
  journal={Machine Intelligence Research},
  volume={20},
  number={2},
  pages={233--248},
  year={2023},
  publisher={Springer}
}

@article{lyu2021contrasting,
  title={Contrasting centralized and decentralized critics in multi-agent reinforcement learning},
  author={Lyu, Xueguang and Xiao, Yuchen and Daley, Brett and Amato, Christopher},
  journal={arXiv preprint arXiv:2102.04402},
  year={2021}
}

@article{jennings2001automated,
  title={Automated negotiation: prospects, methods and challenges},
  author={Jennings, Nicholas R and Faratin, Peyman and Lomuscio, Alessio R and Parsons, Simon and Sierra, Carles and Wooldridge, Michael},
  journal={International Journal of Group Decision and Negotiation},
  volume={10},
  number={2},
  pages={199--215},
  year={2001}
}

@book{kraus2001strategic,
  title={Strategic negotiation in multiagent environments},
  author={Kraus, Sarit},
  year={2001},
  publisher={MIT press}
}

@article{roughgarden2010algorithmic,
  title={Algorithmic game theory},
  author={Roughgarden, Tim},
  journal={Communications of the ACM},
  volume={53},
  number={7},
  pages={78--86},
  year={2010},
  publisher={ACM New York, NY, USA}
}

@article{milgrom1982theory,
  title={A theory of auctions and competitive bidding},
  author={Milgrom, Paul R and Weber, Robert J},
  journal={Econometrica: Journal of the Econometric Society},
  pages={1089--1122},
  year={1982},
  publisher={JSTOR}
}

@article{kamien1990value,
  title={On the value of information in a strategic conflict},
  author={Kamien, Morton I and Tauman, Yair and Zamir, Shmuel},
  journal={Games and Economic Behavior},
  volume={2},
  number={2},
  pages={129--153},
  year={1990},
  publisher={Elsevier}
}

@book{aumann1995repeated,
  title={Repeated games with incomplete information},
  author={Aumann, Robert J and Maschler, Michael and Stearns, Richard E},
  year={1995},
  publisher={MIT press}
}

@article{harsanyi1967games,
  title={Games with incomplete information played by “Bayesian” players, I--III Part I. The basic model},
  author={Harsanyi, John C},
  journal={Management science},
  volume={14},
  number={3},
  pages={159--182},
  year={1967},
  publisher={INFORMS}
}

@article{myerson1981optimal,
  title={Optimal auction design},
  author={Myerson, Roger B},
  journal={Mathematics of operations research},
  volume={6},
  number={1},
  pages={58--73},
  year={1981},
  publisher={INFORMS}
}

@article{vickrey1961counterspeculation,
  title={Counterspeculation, auctions, and competitive sealed tenders},
  author={Vickrey, William},
  journal={The Journal of finance},
  volume={16},
  number={1},
  pages={8--37},
  year={1961},
  publisher={JSTOR}
}

@article{shapley1953stochastic,
  title={Stochastic games},
  author={Shapley, Lloyd S},
  journal={Proceedings of the national academy of sciences},
  volume={39},
  number={10},
  pages={1095--1100},
  year={1953},
  publisher={National Academy of Sciences}
}

@article{wu2019behavior,
  title={Behavior regularized offline reinforcement learning},
  author={Wu, Yifan and Tucker, George and Nachum, Ofir},
  journal={arXiv preprint arXiv:1911.11361},
  year={2019}
}

@article{sion1958general,
  title={On general minimax theorems.},
  author={Sion, Maurice},
  year={1958}
}

@article{kartik2021upper,
  title={Upper and lower values in zero-sum stochastic games with asymmetric information},
  author={Kartik, Dhruva and Nayyar, Ashutosh},
  journal={Dynamic Games and Applications},
  volume={11},
  number={2},
  pages={363--388},
  year={2021},
  publisher={Springer}
}

@article{freund1999adaptive,
  title={Adaptive game playing using multiplicative weights},
  author={Freund, Yoav and Schapire, Robert E},
  journal={Games and Economic Behavior},
  volume={29},
  number={1-2},
  pages={79--103},
  year={1999},
  publisher={Elsevier}
}

@incollection{littman1994markov,
  title={Markov games as a framework for multi-agent reinforcement learning},
  author={Littman, Michael L},
  booktitle={Machine learning proceedings 1994},
  pages={157--163},
  year={1994},
  publisher={Elsevier}
}

@inproceedings{zhong2022pessimistic,
  title={Pessimistic minimax value iteration: Provably efficient equilibrium learning from offline datasets},
  author={Zhong, Han and Xiong, Wei and Tan, Jiyuan and Wang, Liwei and Zhang, Tong and Wang, Zhaoran and Yang, Zhuoran},
  booktitle={International Conference on Machine Learning},
  pages={27117--27142},
  year={2022},
  organization={PMLR}
}

@article{cui2022offline,
  title={When are offline two-player zero-sum Markov games solvable?},
  author={Cui, Qiwen and Du, Simon S},
  journal={Advances in Neural Information Processing Systems},
  volume={35},
  pages={25779--25791},
  year={2022}
}

@inproceedings{jin2021pessimism,
  title={Is pessimism provably efficient for offline rl?},
  author={Jin, Ying and Yang, Zhuoran and Wang, Zhaoran},
  booktitle={International conference on machine learning},
  pages={5084--5096},
  year={2021},
  organization={PMLR}
}

@article{xie2021bellman,
  title={Bellman-consistent pessimism for offline reinforcement learning},
  author={Xie, Tengyang and Cheng, Ching-An and Jiang, Nan and Mineiro, Paul and Agarwal, Alekh},
  journal={Advances in neural information processing systems},
  volume={34},
  pages={6683--6694},
  year={2021}
}

@inproceedings{zhang2023offline,
  title={Offline learning in markov games with general function approximation},
  author={Zhang, Yuheng and Bai, Yu and Jiang, Nan},
  booktitle={International Conference on Machine Learning},
  pages={40804--40829},
  year={2023},
  organization={PMLR}
}

@inproceedings{xie2021batch,
    author = "Xie, Tengyang and Jiang, Nan",
    year = "2021",
    booktitle = "Proceedings of the International Conference on Machine Learning",
    title = "Batch value-function approximation with only realizability"
}

@inproceedings{kumar2020conservative,
    author = "Kumar, Aviral and Zhou, Aurick and Tucker, George and Levine, Sergey",
    year = "2020",
    booktitle = "Advances in Neural Information Processing Systems",
    title = "Conservative {Q}-Learning for Offline Reinforcement Learning"
}

@inproceedings{foerster2018counterfactual,
    author = "Foerster, Jakob N. and Farquhar, Gregory and Afouras, Triantafyllos and Nardelli, Nantas and Whiteson, Shimon",
    title = "Counterfactual Multi-Agent Policy Gradients",
    booktitle = "Proceedings of the {AAAI} Conference on Artificial Intelligence",
    year = "2018"
}

@article{liu2024doubly,
    author = "Liu, Shuze and Chen, Claire and Zhang, Shangtong",
    year = "2024",
    journal = "ArXiv Preprint",
    title = "Doubly Optimal Policy Evaluation for Reinforcement Learning"
}

@article{dudik2011doubly,
    author = "Dud{\'\i}k, Miroslav and Langford, John and Li, Lihong",
    year = "2011",
    journal = "ArXiv Preprint",
    title = "Doubly robust policy evaluation and learning"
}

@inproceedings{liu2024efficientmul,
    author = "Liu, Shuze and Chen, Yuxin and Zhang, Shangtong",
    title = "Efficient Multi-Policy Evaluation for Reinforcement Learning",
    booktitle = "Proceedings of the {AAAI} International Conference on Artificial Intelligence",
    year = "2025"
}

@inproceedings{liu2024efficient,
    author = "Liu, Shuze and Zhang, Shangtong",
    year = "2024",
    booktitle = "Proceedings of the International Conference on Machine Learning",
    title = "Efficient Policy Evaluation with Offline Data Informed Behavior Policy Design"
}

@article{chen2024efficient,
    author = "Chen, Claire and Liu, Shuze and Zhang, Shangtong",
    year = "2024",
    journal = "ArXiv Preprint",
    title = "Efficient Policy Evaluation with Safety Constraint for Reinforcement Learning"
}

@inproceedings{lowe2017multi,
    author = "Lowe, Ryan and Wu, Yi and Tamar, Aviv and Harb, Jean and Abbeel, Pieter and Mordatch, Igor",
    year = "2017",
    booktitle = "Advances in Neural Information Processing Systems",
    title = "Multi-Agent Actor-Critic for Mixed Cooperative-Competitive Environments"
}

@inproceedings{fujimoto2019off,
    author = "Fujimoto, Scott and Meger, David and Precup, Doina",
    year = "2019",
    booktitle = "Proceedings of the International Conference on Machine Learning",
    title = "Off-Policy Deep Reinforcement Learning without Exploration"
}

@article{levine2020offline,
    author = "Levine, Sergey and Kumar, Aviral and Tucker, George and Fu, Justin",
    year = "2020",
    journal = "ArXiv Preprint",
    title = "Offline Reinforcement Learning: Tutorial, Review, and Perspectives on Open Problems"
}

@article{liu2025ode,
    author = "Liu, Shuze and Chen, Shuhang and Zhang, Shangtong",
    title = "The {ODE} Method for Stochastic Approximation and Reinforcement Learning with Markovian Noise",
    year = "2025",
    journal = "Journal of Machine Learning Research"
}

@article{cui2022provably,
  title={Provably efficient offline multi-agent reinforcement learning via strategy-wise bonus},
  author={Cui, Qiwen and Du, Simon S},
  journal={Advances in Neural Information Processing Systems},
  volume={35},
  pages={11739--11751},
  year={2022}
}

@article{uehara2021pessimistic,
  title={Pessimistic model-based offline reinforcement learning under partial coverage},
  author={Uehara, Masatoshi and Sun, Wen},
  journal={arXiv preprint arXiv:2107.06226},
  year={2021}
}

@article{rashidinejad2021bridging,
  title={Bridging offline reinforcement learning and imitation learning: A tale of pessimism},
  author={Rashidinejad, Paria and Zhu, Banghua and Ma, Cong and Jiao, Jiantao and Russell, Stuart},
  journal={Advances in Neural Information Processing Systems},
  volume={34},
  pages={11702--11716},
  year={2021}
}

\clearpage
\appendix
\section{Optimization via No-Regret Dynamics}
\label{sec:practical_optimization}

\subsection{Iterative Minimax Optimization}

 \texttt{PPA-Exact} solves each contextual pessimistic game by linear programming. We also consider an iterative alternative based on no-regret dynamics \citep{cesa2006prediction}. The public--private structure decomposes the minimax objective into local decision problems, so each player can update its policy by multiplicative weights \citep{arora2012multiplicative} conditioned on its private observation. This yields an $O(T^{-1/2})$ optimization guarantee and provides the optimization framework used in Section~\ref{sec:function_approximation}.

\subsection{Local Q-Values and Minimax Decomposition}
For a fixed public context $c \in \mathcal{C}$, the expected pessimistic objective $L_c(\pi_1, \pi_2)$ can be rewritten in terms of local $Q$-values.

When Player 2 plays policy $\pi_2$, the conditional expected pessimistic reward for Player 1 (their local $Q$-value) is:
\begin{equation}
 Q_1(x^1, a^1; \pi_2) \defeq \sum_{x^2 \in \mathcal{X}_2, a^2 \in \mathcal{A}_2} \wh{b}(x^2 \mid c, x^1) \pi_2(a^2 \mid c, x^2) \underline{r}(c, x^1, x^2, a^1, a^2).
 \label{eq:local_q1}
\end{equation}
Similarly, Player 2's local $Q$-value against $\pi_1$ is:
\begin{equation}
 Q_2(x^2, a^2; \pi_1) \defeq \sum_{x^1 \in \mathcal{X}_1, a^1 \in \mathcal{A}_1} \wh{b}(x^1 \mid c, x^2) \pi_1(a^1 \mid c, x^1) \underline{r}(c, x^1, x^2, a^1, a^2).
 \label{eq:local_q2}
\end{equation}
By simulating simultaneous no-regret dynamics using Multiplicative Weights Updates independently at each private state, we can recover an $O(1/\sqrt{T})$-approximate saddle point of the empirical pessimistic objective $L_c$, for the time-averaged iterates, without ever explicitly constructing the full prescriptive LP.
 We demonstrate this method in Algorithm~\ref{alg:mwu_solver}.

\begin{algorithm}[h]
\caption{Public-Private Asymmetric No-Regret Learning (PPA-Learning)}
\label{alg:mwu_solver}
\begin{algorithmic}[1]
\Require Public context $c$, empirical belief $\wh{b}$, clipped pessimistic rewards $\underline{r}$, iterations $T$.
\State \textbf{Initialize:} Uniform policies $\pi_{1,1}(a^1 \mid c, x^1) = 1/|\mathcal{A}_1|$ and $\pi_{2,1}(a^2 \mid c, x^2) = 1/|\mathcal{A}_2|$.
\State \textbf{Set Learning Rate:} $\eta = \sqrt{ \log(\max\{|\mathcal{A}_1|, |\mathcal{A}_2|\}) / T}$.
\For {$t = 1, \dots, T$}
 \State \textbf{Compute Local $Q$-values:}
 \State \quad For all $x^1, a^1$: compute $Q_{1,t}(x^1, a^1) = Q_1(x^1, a^1; \pi_{2,t})$ via Eq. \eqref{eq:local_q1}
 \State \quad For all $x^2, a^2$: compute $Q_{2,t}(x^2, a^2) = Q_2(x^2, a^2; \pi_{1,t})$ via Eq. \eqref{eq:local_q2}
 \State \textbf{Mirror Descent Policy Update:}
 \State \quad $\pi_{1,t+1}(a^1 \mid c, x^1) \propto \pi_{1,t}(a^1 \mid c, x^1) \exp\left( \eta Q_{1,t}(x^1, a^1) \right)$
 \State \quad $\pi_{2,t+1}(a^2 \mid c, x^2) \propto \pi_{2,t}(a^2 \mid c, x^2) \exp\left( -\eta Q_{2,t}(x^2, a^2) \right)$
\EndFor
\Ensure The average strategy for Player 1: $\bar{\pi}_{1,c} = \frac{1}{T} \sum_{t=1}^T \pi_{1,t}$.
\end{algorithmic}
\end{algorithm}

\subsection{Optimization Guarantee and Unified Exploitability Bound}

 The iterative solver approximates the empirical minimax solution at rate $O(T^{-1/2})$; combining this optimization guarantee with the statistical analysis of \texttt{PPA-Exact} gives the following result.

\begin{theorem}
\label{thm:iterative_gap}
Suppose the dataset size satisfies $n \ge \frac{8}{d^\rho_{\min}} \log\left( \frac{2 |\mathcal{S}| |\mathcal{A}|}{\delta} \right)$. Let $\bar{\pi}_1$ be the policy obtained by substituting Algorithm \ref{alg:mwu_solver} into PPA-Exact, run for $T$ iterations. Under Assumption \ref{ass:unilateral_coverage}, with probability at least $1-\delta$:
\begin{equation}
 Gap(\bar{\pi}_1) \le 2(2 + 2\kappa_b) \sqrt{\frac{C^\star_{\text{Pres}} |\mathcal{S}| |\mathcal{A}| \iota}{n}} + 4 \sqrt{\frac{\log(\max\{|\mathcal{A}_1|, |\mathcal{A}_2|\})}{T}}.
\end{equation}
where $\kappa_b = \frac{\sqrt{|\mathcal{X}_1| |\mathcal{X}_2|} + \sqrt{2}}{\sqrt{2}}$ and $\iota \defeq \log(4|\mathcal{S}||\mathcal{A}|/\delta)$.
\end{theorem}

Its proof is in Appendix~\ref{proof:iterative_gap}. 
\section{Proof}
\label{app: proofs}
\subsection{Coverage Separation: Proofs for Section~\ref{sec:coverage_analysis}}
\label{app: proof relaxation}

Throughout this subsection we adopt the convention that $q(s,a)/d^\rho(s,a) = +\infty$ whenever $q(s,a) > 0 = d^\rho(s,a)$, and $0/0 = 0$; all statements hold verbatim under this convention. Every behavior density appearing in Section~\ref{sec:coverage_analysis} --- in Proposition~\ref{prop:separation} and, for every $t > 0$, the sharpness family of Proposition~\ref{prop:integrated_equality} --- has full support, so the convention is never invoked there; it is stated because $C_\Sigma$ in Proposition~\ref{prop:integrated_equality} is defined over all of $\Scal \times \Acal$ and would otherwise be ambiguous.

\subsubsection{Proof of Lemma~\ref{lem:fixed_anchor}}
\label{proof:fixed_anchor}

We prove the following precise form of Lemma~\ref{lem:fixed_anchor}, which also records the common value of the two suprema.
\begin{lemma}[Precise Form of Lemma~\ref{lem:fixed_anchor}]
\label{lem:fixed_anchor_precise}
Fix any reference policy $\bar{\pi}_1: \Scal \to \Delta(\Acal_1)$, prescriptive or fully observed, and read the induced density with the conditioning each policy admits,
\begin{equation*}
d^{\bar{\pi}_1, \nu}(s, a) \defeq \rho(s)\, \bar{\pi}_1(a^1 \mid s)\, \nu(a^2 \mid \cdot),
\qquad
\nu(a^2 \mid \cdot) =
\begin{cases}
\nu(a^2 \mid c, x^2), & \nu \in \Pi_2, \\
\nu(a^2 \mid s), & \nu \in \Pi_2^{\text{full}},
\end{cases}
\end{equation*}
so that a prescriptive deviator is embedded in $\Pi_2^{\text{full}}$ by $\nu(a^2 \mid s) \defeq \nu(a^2 \mid c, x^2)$. Then, for an arbitrary logged density $d^\rho$, not necessarily of decentralised product form,
\begin{equation*}
\sup_{\nu \in \Pi_2} \max_{s, a} \frac{d^{\bar{\pi}_1, \nu}(s, a)}{d^\rho(s, a)}
= \sup_{\nu \in \Pi_2^{\text{full}}} \max_{s, a} \frac{d^{\bar{\pi}_1, \nu}(s, a)}{d^\rho(s, a)}
= \max_{s, a} \frac{\rho(s)\, \bar{\pi}_1(a^1 \mid s)}{d^\rho(s, a)},
\end{equation*}
and symmetrically for Player~1's deviations against a fixed $\bar{\pi}_2: \Scal \to \Delta(\Acal_2)$. When $\bar{\pi}_1 \in \Pi_1$ is prescriptive, $\bar{\pi}_1(a^1 \mid s)$ reduces to $\bar{\pi}_1(a^1 \mid c, x^1)$.
\end{lemma}
\begin{proof}
Fix a reference policy $\bar{\pi}_1: \Scal \to \Delta(\Acal_1)$ and set
\begin{equation*}
g(s, a) \defeq \frac{\rho(s)\, \bar{\pi}_1(a^1 \mid s)}{d^\rho(s, a)}, \qquad G \defeq \max_{s \in \Scal, a \in \Acal} g(s, a),
\end{equation*}
where the maximum is attained because $\Scal \times \Acal$ is finite. For \emph{any} $\nu$, prescriptive or full-information, the induced density factorises as $d^{\bar{\pi}_1, \nu}(s, a) = \rho(s) \bar{\pi}_1(a^1 \mid s) \nu(a^2 \mid \cdot)$ --- note that no property of the reference policy beyond this factorisation is used --- so
\begin{equation}
\frac{d^{\bar{\pi}_1, \nu}(s, a)}{d^\rho(s, a)} \;\le\; g(s, a) \;\le\; G.
\end{equation}
Writing $N \defeq \rho(s)\bar{\pi}_1(a^1 \mid s)$, $D \defeq d^\rho(s,a)$ and $t \defeq \nu(a^2 \mid \cdot) \in [0,1]$, this holds in each of the three cases separately: if $D > 0$ the ratio is $g(s,a)\,t \le g(s,a)$; if $D = 0$ and $Nt > 0$ then $N > 0$, so $g(s,a) = G = +\infty$; and if $D = 0$ and $Nt = 0$ the ratio is $0$ by the convention $0/0 = 0$. Taking $\max_{s,a}$ and then the supremum over either policy class gives the upper bound $G$ for both.

For the matching lower bound, let $(s^\dagger, a^\dagger) = (c^\dagger, x^{1\dagger}, x^{2\dagger}, a^{1\dagger}, a^{2\dagger})$ attain $G$, and define $\nu^\dagger$ by
\begin{equation*}
\nu^\dagger(\cdot \mid c^\dagger, x^{2\dagger}) = \delta_{a^{2\dagger}}, \qquad \nu^\dagger(\cdot \mid c, x^2) \text{ arbitrary for } (c, x^2) \ne (c^\dagger, x^{2\dagger}).
\end{equation*}
Crucially $\nu^\dagger \in \Pi_2$: it conditions only on the local observation $(c, x^2)$, and its behaviour at the hidden types $x^1 \ne x^{1\dagger}$ that share the observation $(c^\dagger, x^{2\dagger})$ is irrelevant, because the criterion is a maximum over a \emph{single} state-action tuple. Then
\begin{equation*}
\sup_{\nu \in \Pi_2} \max_{s, a} \frac{d^{\bar{\pi}_1, \nu}(s, a)}{d^\rho(s, a)} \;\ge\; \frac{d^{\bar{\pi}_1, \nu^\dagger}(s^\dagger, a^\dagger)}{d^\rho(s^\dagger, a^\dagger)} = g(s^\dagger, a^\dagger) \cdot 1 = G.
\end{equation*}
Hence the prescriptive supremum equals $G$. Since $\Pi_2 \subseteq \Pi_2^{\text{full}}$, the full-information supremum is sandwiched between $G$ and $G$ as well. If $G = +\infty$ the same $\nu^\dagger$ makes both sides $+\infty$. The argument for Player~1's deviations against a fixed $\bar{\pi}_2: \Scal \to \Delta(\Acal_2)$ is identical after exchanging the roles of the players, with $\mu^\dagger(\cdot \mid c^\dagger, x^{1\dagger}) = \delta_{a^{1\dagger}}$. Note that no property of $d^\rho$ beyond nonnegativity was used, and nothing about which coordinates of $s$ the reference policy is measurable with respect to; this is what makes the statement cover fully observed reference policies as well.
\end{proof}

\subsubsection{An integrated coverage criterion}

The following proposition is the robustness check referred to in Section~\ref{sec:coverage_analysis}: it shows that the equality of Lemma~\ref{lem:fixed_anchor} is not an artifact of the pointwise $L_\infty$ criterion.

\begin{proposition}[The equality at a common reference policy is not an artifact of the pointwise criterion]
\label{prop:integrated_equality}
For a deviation class $P$ define the integrated coefficient
$C_\Sigma(\bar{\pi}_1; P) \defeq \sup_{\nu \in P} \sum_{s, a} d^{\bar{\pi}_1, \nu}(s, a) / d^\rho(s, a)$,
the sum running over all of $\Scal \times \Acal$ under Convention~1 of Appendix~\ref{proof:finite_rate}, so that target mass placed off $\mathrm{supp}(d^\rho)$ makes the coefficient $+\infty$. For an arbitrary behavior density $d^\rho$, not necessarily of decentralised product form, the two coefficients are infinite simultaneously, and whenever they are finite
\begin{equation*}
C_\Sigma(\bar{\pi}_1; \Pi_2) \;\le\; C_\Sigma(\bar{\pi}_1; \Pi_2^{\text{full}}) \;\le\; \min\{|\Xcal_1|, |\Acal_2|\} \cdot C_\Sigma(\bar{\pi}_1; \Pi_2),
\end{equation*}
where $\min\{|\Xcal_1|, |\Acal_2|\}$ cannot be improved: for every $k \ge 2$ there are instances with $|\Xcal_1| = |\Acal_2| = k$ and a family of \emph{full-support} behavior densities along which the ratio tends to $k$. However, whenever the logged behavior is itself decentralised, $d^\rho(s, a) = \rho(s)\mu_1(a^1 \mid c, x^1)\mu_2(a^2 \mid c, x^2)$, the two coincide:
\begin{equation*}
C_\Sigma(\bar{\pi}_1; \Pi_2) = C_\Sigma(\bar{\pi}_1; \Pi_2^{\text{full}}) \qquad \text{for every reference policy } \bar{\pi}_1 \text{ and every } \rho, \mu_1, \mu_2.
\end{equation*}
\end{proposition}

\begin{proof}
Fix any reference policy $\bar{\pi}_1: \Scal \to \Delta(\Acal_1)$, prescriptive or fully observed, and write $g(s,a) = \rho(s)\bar{\pi}_1(a^1 \mid s)/d^\rho(s,a)$, read under Convention~1 and with the mixed-pair conditioning of Lemma~\ref{lem:fixed_anchor}. Both coefficients are $+\infty$ exactly when some $(s,a)$ has $\rho(s)\bar{\pi}_1(a^1 \mid s) > 0 = d^\rho(s,a)$: a prescriptive $\nu$ can place unit mass on that $a^2$ at $(c, x^2)$, and a full-information $\nu$ can do no more, so the two are infinite together and the asserted identities hold trivially. Assume henceforth that both are finite, so that $g$ is real-valued wherever the reference policy puts mass. Since the objective $\sum_{s,a} g(s,a) \nu(a^2 \mid \cdot)$ is linear in $\nu$ over a product of simplices, the supremum is attained at a deterministic $\nu$, and grouping the sum by the coordinates each class may condition on gives the closed forms
\begin{align}
C_\Sigma(\bar{\pi}_1; \Pi_2) &= \sum_{c, x^2} \max_{a^2} \sum_{x^1} \sum_{a^1} g(c, x^1, x^2, a^1, a^2), \\
C_\Sigma(\bar{\pi}_1; \Pi_2^{\text{full}}) &= \sum_{c, x^1, x^2} \max_{a^2} \sum_{a^1} g(c, x^1, x^2, a^1, a^2).
\end{align}
Since a maximum of a sum never exceeds the sum of the maxima, $C_\Sigma(\bar{\pi}_1; \Pi_2) \le C_\Sigma(\bar{\pi}_1; \Pi_2^{\text{full}})$, and writing $M(x^1) = \max_{a^2} \sum_{a^1} g$ and $S = \max_{a^2} \sum_{x^1, a^1} g$ we have $M(x^1) \le S$ for each $x^1$, so the ratio is at most $|\Xcal_1|$. It is also at most $|\Acal_2|$: choosing $j(x^1) \in \argmax_{a^2} \sum_{a^1} g$ and grouping the hidden types by their selected action,
\begin{equation*}
\sum_{x^1} M(x^1) = \sum_{a^2 \in \mathrm{image}(j)} \; \sum_{x^1: j(x^1) = a^2} M(x^1) \le |\mathrm{image}(j)| \cdot S \le \min\{|\Xcal_1|, |\Acal_2|\} \cdot S,
\end{equation*}
since for each fixed $a^2$ the inner sum is at most $S$ by definition of $S$.

The factor $\min\{|\Xcal_1|, |\Acal_2|\}$ cannot be improved. Take $|\Ccal| = 1$, $|\Xcal_2| = 1$, $|\Xcal_1| = k \ge 2$, $\Acal_1 = \{\alpha\}$ a singleton (so $\bar{\pi}_1$ is forced), $\Acal_2 = [k]$, and $\rho$ uniform on $\Xcal_1$, so that $\min\{|\Xcal_1|, |\Acal_2|\} = k$. For $t \in (0, 1/k]$ let the logged density concentrate on the graph $a^2 = x^1$ with weight $t$ and spread the rest uniformly,
\begin{equation}
d^\rho_t(x^1, \alpha, a^2) = \frac{1}{k}\left[ t\, \mathbf{1}\{a^2 = x^1\} + \frac{1 - t}{k-1}\, \mathbf{1}\{a^2 \ne x^1\} \right],
\label{eq:witness_logging}
\end{equation}
which is a probability distribution with \emph{full support}, since $t + (k-1)\frac{1-t}{k-1} = 1$ for each $x^1$. Here $g = \rho(x^1)\bar{\pi}_1(\alpha)/d^\rho_t$, so $\sum_{a^1} g$ equals $1/t$ at $a^2 = x^1$ and $(k-1)/(1-t)$ otherwise. The restriction $t \le 1/k$ is exactly the condition $1/t \ge (k-1)/(1-t)$ that makes the diagonal entry the largest, whence
\begin{equation*}
C_\Sigma(\bar{\pi}_1; \Pi_2^{\text{full}}) = \sum_{x^1} \max_{a^2} \sum_{a^1} g = \frac{k}{t}, \qquad
C_\Sigma(\bar{\pi}_1; \Pi_2) = \max_{a^2} \sum_{x^1} \sum_{a^1} g = \frac{1}{t} + \frac{(k-1)^2}{1-t},
\end{equation*}
the latter because the inner sum takes the same value for every $a^2$. Therefore
\begin{equation}
\frac{C_\Sigma(\bar{\pi}_1; \Pi_2^{\text{full}})}{C_\Sigma(\bar{\pi}_1; \Pi_2)} = \frac{k}{1 + \dfrac{t (k-1)^2}{1-t}} \; \longrightarrow \; k = \min\{|\Xcal_1|, |\Acal_2|\} \quad \text{as } t \to 0^+,
\label{eq:witness_ratio}
\end{equation}
strictly below $k$ for every $t \in (0, 1/k]$, and tending to $k$.

Now impose $d^\rho(s,a) = \rho(s)\mu_1(a^1 \mid c, x^1)\mu_2(a^2 \mid c, x^2)$. On $\mathrm{supp}(d^\rho)$ the prior cancels and
\begin{equation*}
g(s, a) = \frac{\bar{\pi}_1(a^1 \mid s)}{\mu_1(a^1 \mid c, x^1)\, \mu_2(a^2 \mid c, x^2)},
\end{equation*}
so, setting $G(s) \defeq \sum_{a^1: \mu_1(a^1 \mid c, x^1) > 0} \bar{\pi}_1(a^1 \mid s)/\mu_1(a^1 \mid c, x^1) \ge 0$ and $h(c, x^2, a^2) \defeq \mathbf{1}\{\mu_2(a^2 \mid c, x^2) > 0\}/\mu_2(a^2 \mid c, x^2)$, the inner sum factorises as
\begin{equation}
\sum_{a^1} g(c, x^1, x^2, a^1, a^2) = \mathbf{1}\{\rho(c, x^1, x^2) > 0\} \, G(s) \, h(c, x^2, a^2).
\label{eq:rank_one}
\end{equation}
The dependence on $a^2$ enters only through $h$, which does not depend on $x^1$. Hence for every $x^1$ the maximisers over $a^2$ are the same --- namely the least-logged actions $\argmin_{a^2: \mu_2 > 0} \mu_2(a^2 \mid c, x^2)$ --- and, because $\mathbf{1}\{\rho > 0\} G(s) \ge 0$ does not depend on $a^2$, it can be taken outside the maximization over $a^2$, 
\begin{equation*}
\sum_{x^1} \max_{a^2} \sum_{a^1} g = \Bigl[ \sum_{x^1} \mathbf{1}\{\rho > 0\} G(s) \Bigr] \max_{a^2} h(c, x^2, a^2) = \max_{a^2} \sum_{x^1} \sum_{a^1} g
\end{equation*}
for every $(c, x^2)$. Summing over $(c, x^2)$ gives the claimed equality.
\end{proof}

\subsubsection{Proof of Proposition~\ref{prop:separation}}
\label{proof:separation}

Proposition~\ref{prop:separation} is the case $\mu(0) = 1/2$, $\mu(j) = 1/(2K)$ of the following statement, which we prove.
\begin{proposition}[Coverage Separation for General Symmetric Loggers]
\label{prop:separation_general}
For every integer $K \ge 3$ there is a public--private zero-sum game, constructed in the proof below, whose prescriptive and fully observed games each admit a unique Nash equilibrium, and whose equilibria do not depend on the logged behavior. For \emph{any} symmetric, type-independent behavior policy $\mu_1 = \mu_2 = \mu$ with full support on $\{0, 1, \dots, K\}$,
\begin{equation}
C^\star_{\text{Pres}} = \frac{1}{\mu(0)\, \min_{a} \mu(a)}, \qquad
C^\star_{\text{Full}} = \frac{1}{\min_{j \in [K]} \mu(j) \cdot \min_{a} \mu(a)}, \qquad
\frac{C^\star_{\text{Full}}}{C^\star_{\text{Pres}}} = \frac{\mu(0)}{\min_{j \in [K]} \mu(j)}.
\label{eq:alignment}
\end{equation}
In particular, $\mu(0) = 1/2$ and $\mu(j) = 1/(2K)$ give $C^\star_{\text{Pres}} = 4K$ and $C^\star_{\text{Full}} = 4K^2$, an improvement by the factor $K$; while $\mu(0) = 1/(K^2+1)$ and $\mu(j) = K/(K^2+1)$ give $C^\star_{\text{Pres}} = (K^2+1)^2$ and $C^\star_{\text{Full}} = (K^2+1)^2/K$, a deterioration by the same factor. Consequently neither $C^\star_{\text{Pres}} \le C^\star_{\text{Full}}$ nor its reverse holds for every behavior distribution.
\end{proposition}
\begin{proof}
Fix $K \ge 3$. Let $|\Ccal| = 1$, let $X^1, X^2$ be independent and uniform on $[K]$ so that $\rho(x^1, x^2) = 1/K^2$, and let $\Acal_1 = \Acal_2 = \{0, 1, \dots, K\}$. With $u(a, x) = \mathbf{1}\{a = 0\} + 2 \cdot \mathbf{1}\{a = x\}$ we have $u \in \{0, 1, 2\}$, hence
\begin{equation*}
r(x^1, x^2, a^1, a^2) = \tfrac{1}{2}\left[ u(a^1, x^2) - u(a^2, x^1) \right] \in [-1, 1],
\end{equation*}
as required by the model. The logged behavior is $\mu_i(0 \mid x^i) = \mu(0)$ and $\mu_i(j \mid x^i) = \mu(j)$ for $j \in [K]$, with $\mu(0) + \sum_{j \in [K]} \mu(j) = 1$ and $\mu > 0$ everywhere, so $d^\rho$ has full support. We keep $\mu$ general throughout and specialise only at the end.

\textbf{Step 1: both games have a unique equilibrium, independent of $\mu$.}
The reward is additively separable, $r = \tfrac{1}{2} u(a^1, x^2) - \tfrac{1}{2} u(a^2, x^1)$, so each player's optimal action is independent of the opponent's strategy and it suffices to maximise the player's own term. In particular the equilibria depend only on $u$, $r$ and $\rho$, never on the behavior policy.

Under prescriptive execution Player~1 observes only $(c, x^1)$, and since $X^2 \perp X^1$ is uniform on $[K]$,
\begin{equation*}
\EE\left[ u(0, X^2) \right] = 1, \qquad \EE\left[ u(j, X^2) \right] = 2 \, \PP(X^2 = j) = \frac{2}{K} \quad \text{for } j \in [K].
\end{equation*}
Because $K \ge 3$ gives $2/K < 1$, the action $a^1 = 0$ is the strict maximiser at every type, hence strictly dominant in $\Pi_1$. Player~2 minimises $r$, equivalently maximises $u(a^2, X^1)$, so by the same computation $a^2 = 0$ is strictly dominant in $\Pi_2$. The prescriptive game therefore has the unique Nash equilibrium $\pi_{\star}^{\text{Pres}} = (\delta_0, \delta_0)$, and strict dominance rules out mixed equilibria as well.

Under full information Player~1 observes $x^2$ and compares $u(0, x^2) = 1$, $u(x^2, x^2) = 2$ (note $x^2 \in [K]$, so $x^2 \ne 0$), and $u(j, x^2) = 0$ for $j \notin \{0, x^2\}$. Thus $a^1 = x^2$ is strictly dominant, and symmetrically $a^2 = x^1$; the fully observed game has the unique equilibrium $\pi_{\star}^{\text{Full}} = (\delta_{x^2}, \delta_{x^1})$.

\textbf{Step 2: a closed form for both coefficients.}
Because $d^\rho(s, a) = \rho(s) \mu_1(a^1 \mid x^1) \mu_2(a^2 \mid x^2)$, the prior cancels in every density ratio:
\begin{equation}
\frac{d^{\pi_{1, \star}, \nu}(s, a)}{d^\rho(s, a)} = \frac{\pi_{1, \star}(a^1 \mid \cdot)}{\mu_1(a^1 \mid x^1)} \cdot \frac{\nu(a^2 \mid \cdot)}{\mu_2(a^2 \mid x^2)}.
\end{equation}
Taking the supremum over $\nu$ and the maximum over $(s, a)$, and using that for each fixed $(x^1, x^2)$ the deviator may place unit mass on the least-covered action, gives
\begin{equation}
\sup_{\nu} \max_{s, a} \frac{d^{\pi_{1, \star}, \nu}(s, a)}{d^\rho(s, a)} = \max_{x^1, x^2} \left\{ \left[ \max_{a^1} \frac{\pi_{1, \star}(a^1 \mid \cdot)}{\mu_1(a^1 \mid x^1)} \right] \cdot \frac{1}{\min_{a^2} \mu_2(a^2 \mid x^2)} \right\},
\label{eq:sep_factorisation}
\end{equation}
where by Lemma~\ref{lem:fixed_anchor} it is immaterial whether the supremum is taken over $\Pi_2$ or over $\Pi_2^{\text{full}}$, at either equilibrium; for the fully observed equilibrium policy $\pi_{1,\star}^{\text{Full}} \notin \Pi_1$ this is the fully observed case of that lemma. Since $\mu$ is type-independent, the second factor equals $1/\min_a \mu(a)$ for every $x^2$.

The two equilibrium policies differ only in which action the first bracket evaluates $\mu_1$ at. At the prescriptive equilibrium $\pi_{1, \star} = \delta_0$ it is $1/\mu(0)$. At the fully observed equilibrium $\pi_{1, \star}(a^1 \mid s) = \mathbf{1}\{a^1 = x^2\}$ it is $1/\mu(x^2)$, maximised over $x^2 \in [K]$ at $1/\min_{j \in [K]} \mu(j)$. Neither expression depends on $x^1$.

The game is symmetric under exchanging the players, so the second argument of the $\max\{\cdot, \cdot\}$ in \eqref{eq:prescriptive_c} and \eqref{eq:standard_c} takes the same value as the first, and both equilibria are unique, so each outer minimisation is over a singleton. Therefore
\begin{equation*}
C^\star_{\text{Pres}} = \frac{1}{\mu(0) \min_a \mu(a)}, \qquad
C^\star_{\text{Full}} = \frac{1}{\min_{j \in [K]} \mu(j) \cdot \min_a \mu(a)},
\end{equation*}
and the common factor $1/\min_a \mu(a)$ cancels in the ratio, leaving $C^\star_{\text{Full}} / C^\star_{\text{Pres}} = \mu(0) / \min_{j \in [K]} \mu(j)$. This is \eqref{eq:alignment}.

\textbf{Step 3: the two instantiations.}
Take $\mu(0) = 1/2$ and $\mu(j) = 1/(2K)$, a valid policy since $1/2 + K/(2K) = 1$. Then $\min_a \mu(a) = 1/(2K)$, so $C^\star_{\text{Pres}} = 1/[(1/2)(1/(2K))] = 4K$ and $C^\star_{\text{Full}} = 1/[(1/(2K))(1/(2K))] = 4K^2$, a ratio of $K$.

Now take $\mu(0) = 1/(K^2+1)$ and $\mu(j) = K/(K^2+1)$, valid since $1/(K^2+1) + K^2/(K^2+1) = 1$. Here $\mu(0) < \mu(j)$ for $K \ge 2$, so $\min_a \mu(a) = \mu(0) = 1/(K^2+1)$, giving $C^\star_{\text{Pres}} = (K^2+1)^2$ and $C^\star_{\text{Full}} = (K^2+1)^2/K$, a ratio of $1/K$.

 Hence $C^\star_{\text{Full}}/C^\star_{\text{Pres}}$ equals $K$ in the first case and $1/K$ in the second, so neither ordering holds universally.
\end{proof}

\subsection{Proof of Theorem~\ref{thm:finite_rate}}
\label{proof:finite_rate}
\begin{proof}
 \textbf{Convention 1 (density ratios).} For a target distribution $q$ we set $q(s,a)/d^\rho(s,a) \defeq +\infty$ if $q(s,a) > 0 = d^\rho(s,a)$ and $\defeq 0$ if $q(s,a) = d^\rho(s,a) = 0$. Finiteness of $C^\star_{\text{Pres}}$ then implies $d^{\pi_{1,\star}, \nu} \ll d^\rho$ for every $\nu \in \Pi_2$, and symmetrically, so sums over $\Scal \times \Acal$ may be restricted to $\mathrm{supp}(d^\rho)$.

\textbf{Convention 2 (unobserved contexts).} If $N(c) = 0$, then $\wh{b}(\cdot \mid c)$ is the fallback $b_0(\cdot \mid c)$ of Section~\ref{sec:finite_learning}.

Throughout, $\iota = \log(4|\Scal||\Acal|/\delta) \ge \log 4 > 1$.

Note that here we treat the logged behavior distribution \(d^\rho\) as fixed and ask whether it covers the equilibrium and unilateral deviations needed for offline learning. This is complementary to work that actively designs data-collection policies for lower-variance, safer, or more robust policy evaluation \citep{liu2024efficient,liu2024doubly,chen2024efficient,liu2024efficientmul,chen2026robust}; see also \citet{liu2025efficientrobust} for a broader treatment.

Before proving Theorem~\ref{thm:finite_rate}, we present and prove the following lemma.
\begin{lemma}[$L_1$ Concentration of Empirical Beliefs]
\label{lem: L1 concentration of empirical beliefs}
For a fixed public context $c \in \Ccal$, let $b(\cdot \mid c)$ be the true distribution over $\Xcal_1 \times \Xcal_2$ and $\wh{b}(\cdot \mid c)$ be the empirical distribution derived from $n = N(c)$ i.i.d. samples, where the statement is to be read \emph{conditionally on the realised context-count vector} $(N(c))_{c \in \Ccal}$; see the remark following the proof. For any $\delta' \in (0, 1)$, it holds with probability at least $1-\delta'$ that:

\begin{equation}
 \|b(\cdot \mid c) - \wh{b}(\cdot \mid c)\|_1 \le \Ecal_b(c) \defeq
 \begin{cases}
 \sqrt{\dfrac{|\Xcal_1| |\Xcal_2|}{n}} + \sqrt{\dfrac{2 \log(1/\delta')}{n}}, & n > 0, \\[10pt]
 2, & n = 0.
 \end{cases}
\end{equation}

\end{lemma}

\begin{proof}[Proof of Lemma \ref{lem: L1 concentration of empirical beliefs}]
Let $K = |\Xcal_1| |\Xcal_2|$ be the cardinality of the joint private state space. Consider the function $f(X_1, \dots, X_n) = \|b - \wh{b}\|_1$, where $X_i \in \{1, \dots, K\}$.

\textbf{Step 1: Sensitivity Analysis.} Changing a single sample $X_i$ to $X_i'$ alters the empirical frequency of at most two categories by $1/n$. Thus, the $L_1$ norm satisfies the bounded differences property with $c_i = 2/n$:
\begin{equation}
 |f(X_1, \dots, X_i, \dots, X_n) - f(X_1, \dots, X_i', \dots, X_n)| \le \frac{1}{n} + \frac{1}{n} = \frac{2}{n}.
\end{equation}

\textbf{Step 2: Concentration around the Expectation.} By McDiarmid's Inequality, for any $\epsilon > 0$:
\begin{equation}
 \P\left( f(X) - \E[f(X)] \ge \epsilon \right) \le \exp\left( -\frac{2\epsilon^2}{\sum_{i=1}^n c_i^2} \right) = \exp\left( -\frac{2\epsilon^2}{n(4/n^2)} \right) = \exp\left( -\frac{n\epsilon^2}{2} \right).
\end{equation}
 Taking $\epsilon = \sqrt{2 \log(1/\delta')/n}$ makes the tail bound equal to $\delta'$.

\textbf{Step 3: Bounding the Expected $L_1$ Deviation.} It is a standard result for empirical distributions on a discrete set of size $K$ that $\E[\|b - \wh{b}\|_1] \le \sqrt{\frac{K-1}{n}} \le \sqrt{\frac{K}{n}}$.

\textbf{Step 4: Exact High-Probability Bound.} Combining the expectation bound and the concentration term via the triangle inequality, we obtain the exact bound:
\begin{equation}
 \|b - \wh{b}\|_1 \le \sqrt{\frac{K}{n}} + \sqrt{\frac{2 \log(1/\delta')}{n}}.
\end{equation}
For $n = N(c) = 0$ no sampling has occurred and $\wh{b}(\cdot \mid c) = b_0(\cdot \mid c)$ by Convention~2; there the bound is the trivial one, $\|b - \wh{b}\|_1 \le 2$, which is why $\Ecal_b(c)$ is defined separately in that case. This serves as our exact definition for $\Ecal_b(c)$.
\end{proof}

 \begin{remark}[Removing the conditioning]
Conditional on the context counts $(N(c'))_{c' \in \Ccal}$, the type pairs observed at each context $c$ are i.i.d.\ draws from $b(\cdot \mid c)$, so the lemma applies conditionally. A union bound over $c \in \Ccal$ at level $\delta' = \delta/(4|\Ccal|)$ then holds for every realisation of the counts, and hence unconditionally.
\end{remark}

To ensure a global failure probability of at most $\delta$, we allocate $\delta/2$ to belief estimation and $\delta/2$ to the count-concentration event. Let $\iota = \log(4|\Scal||\Acal|/\delta)$.

By Lemma~\ref{lem: L1 concentration of empirical beliefs}, we have $\|b(\cdot \mid c) - \wh{b}(\cdot \mid c)\|_1 \le \Ecal_b(c)$. Recall that $\beta(s, a) = \sqrt{2} \sqrt{\frac{\iota}{N(s, a) \vee 1}}$. By setting the belief concentration failure probability to $\delta' = \frac{\delta}{4|\Ccal|}$, we ensure the union bound holds and that $\log(1/\delta') \le \iota$, the latter because $|\Scal| |\Acal| = |\Ccal| |\Xcal_1| |\Xcal_2| |\Acal| \ge |\Ccal|$. We bound the context-level belief error by the point-wise uncertainty bonus in two cases. If $N(c) > 0$, then $N(c) \ge N(s, a)$ for every latent state-action pair whose public component is $c$, and $\iota > 1$ gives $\sqrt{|\Xcal_1||\Xcal_2|} \le \sqrt{|\Xcal_1||\Xcal_2|\,\iota}$, so we can reliably bound the context-level belief error using the point-wise uncertainty bonus:
\begin{equation}
 \Ecal_b(c) \le \left( \sqrt{|\Xcal_1| |\Xcal_2|} + \sqrt{2} \right) \sqrt{\frac{\iota}{N(s, a) \vee 1}} = \frac{\sqrt{|\Xcal_1| |\Xcal_2|} + \sqrt{2}}{\sqrt{2}} \beta(s, a) \defeq \kappa_b \beta(s, a), \label{eq: point wise c to AS bound}
\end{equation}
for all $s \in \Scal, a \in \Acal$. If instead $N(c) = 0$, then $N(s, a) = 0$ at every such pair, so $\beta(s,a) = \sqrt{2\iota}$ and $\kappa_b \beta(s,a) = (\sqrt{|\Xcal_1||\Xcal_2|} + \sqrt{2})\sqrt{\iota} \ge (1 + \sqrt{2})\sqrt{\iota} > 2 = \Ecal_b(c)$, using only $|\Xcal_1||\Xcal_2| \ge 1$ and $\iota > 1$; so \eqref{eq: point wise c to AS bound} holds in both cases. Here, we defined
\begin{align}
 \kappa_b\doteq \frac{\sqrt{|\Xcal_1| |\Xcal_2|} + \sqrt{2}}{\sqrt{2}}.
\end{align}

To establish the \emph{approximate pessimism} property of the objective $L_c$ relative to the true expected reward $r_c$, we utilize the clipped pessimistic reward $\underline{r}(s, a) = \max\{\wh{r}(s, a) - \beta(s, a), -1\}$. Because the true reward satisfies $r \in [-1, 1]$, this projection ensures that $\|\underline{r}\|_\infty \le 1$.
The required pessimism property holds \emph{deterministically}, with no reward-concentration event, because the model of Section~\ref{sec:preliminaries} specifies rewards through a function $r: \Scal \times \Acal \to [-1, 1]$.

\begin{lemma}[Deterministic Pessimism of the Clipped Reward]
\label{lem:det_pessimism}
Assume rewards are deterministic, so that $r^{(k)} = r(s^{(k)}, a^{(k)})$ for every logged transition, and $\iota = \log(4|\Scal||\Acal|/\delta) \ge 1$. Then for every $(s, a) \in \Scal \times \Acal$, surely,
\begin{equation*}
\underline{r}(s, a) \le r(s, a) \qquad \text{and} \qquad 0 \le r(s, a) - \underline{r}(s, a) \le 2\beta(s, a).
\end{equation*}
\end{lemma}

\begin{proof}
If $N(s, a) > 0$, determinism gives $\wh{r}(s, a) = r(s, a)$ exactly, so $\wh{r}(s, a) - \beta(s, a) \le r(s, a)$; clipping at $-1$ preserves this because $r(s, a) \ge -1$, and $r(s, a) - \underline{r}(s, a) \le \beta(s, a) \le 2\beta(s, a)$. If $N(s, a) = 0$, then $\wh{r}(s, a) = 0$ and $\beta(s, a) = \sqrt{2\iota} > 1$, so $\wh{r}(s,a) - \beta(s,a) < -1$ and hence $\underline{r}(s, a) = -1 \le r(s, a)$, while $r(s, a) - \underline{r}(s, a) = r(s, a) + 1 \le 2 \le 2\beta(s, a)$. Note $\iota \ge \log 4 > 1$ always holds, since $\delta < 1$ and $|\Scal||\Acal| \ge 1$.
\end{proof}

\noindent Consequently no failure probability need be budgeted for the reward estimate; the entire $\delta/2$ statistical budget is available to the belief estimate.

For any policy pair $(\pi_1, \pi_2)$:
\begin{align}
 r_c(\pi_1, \pi_2) - L_c(\pi_1, \pi_2) &= \sum_{x, a} b r \pi_1 \pi_2 - \sum_{x, a} \wh{b} \underline{r} \pi_1 \pi_2 \label{eq: approximate pessimism decomposition}\\
 &= \sum_{x, a} b (r - \underline{r}) \pi_1 \pi_2 + \sum_{x, a} (b - \wh{b}) \underline{r} \pi_1 \pi_2 \\
 &\le \EE_{s, a \sim d^{\pi_1, \pi_2} \mid c} [2\beta(s, a)] + \|b - \wh{b}\|_1 \|\underline{r}\|_\infty \explain{By $r - \underline{r} \le 2\beta$ and Hölder's} \\
 &\le \EE_{s, a \sim d^{\pi_1, \pi_2} \mid c} [2\beta(s, a)] + \Ecal_b(c). \explain{By $\|b - \wh{b}\|_1 \le \Ecal_b(c)$ and $\|\underline{r}\|_\infty \le 1$}
\end{align}

Recall that the true contextual value is $V^\star(c) = \min_{\pi_2 \in \Pi_2} r_c(\pi_{1, \star}, \pi_2)$. Note that for all $\pi_2 \in \Pi_2$:
\begin{align}
 L_c(\wh{\pi}_1, \pi_2) &= \sum_{x, a} \wh{b} \underline{r} \wh{\pi}_1 \pi_2 \\
 &\le \sum_{x, a} \wh{b} r \wh{\pi}_1 \pi_2 \explain{Since $\underline{r} \le r$} \\
 &= \sum_{x, a} b r \wh{\pi}_1 \pi_2 + \sum_{x, a} (\wh{b} - b) r \wh{\pi}_1 \pi_2 \\
 &\le r_c(\wh{\pi}_1, \pi_2) + \| \wh{b} - b \|_1 \left\| \sum_a r \wh{\pi}_1 \pi_2 \right\|_\infty \explain{Hölder's Inequality} \\
 &\le r_c(\wh{\pi}_1, \pi_2) + \Ecal_b(c) \explain{Since $\| \wh{b} - b \|_1 \le \Ecal_b(c), \|r\|_\infty \le 1$} \\
 \implies \min_{\pi_2 \in \Pi_2} L_c(\wh{\pi}_1, \pi_2) - \Ecal_b(c) &\le \min_{\pi_2 \in \Pi_2} r_c(\wh{\pi}_1, \pi_2) = r_c(\wh{\pi}_1, \nu^\star_c), \explain{Monotonicity of $\min$}
\end{align}
where $\nu^\star_c = \argmin_{\pi_2 \in \Pi_2} r_c(\wh{\pi}_1, \pi_2)$. Therefore we have $-r_c(\wh{\pi}_1, \nu^\star_c) \le -\min_{\pi_2 \in \Pi_2} L_c(\wh{\pi}_1, \pi_2) + \Ecal_b(c)$.

Using \eqref{eq: approximate pessimism decomposition}, the contextual exploitability gap satisfies:
\begin{align}
 V^\star(c) - r_c(\wh{\pi}_1, \nu^\star_c) &\le \min_{\pi_2 \in \Pi_2} r_c(\pi_{1, \star}, \pi_2) - \min_{\pi_2 \in \Pi_2} L_c(\wh{\pi}_1, \pi_2) + \Ecal_b(c) \\
 &\le \min_{\pi_2 \in \Pi_2} r_c(\pi_{1, \star}, \pi_2) - \min_{\pi_2 \in \Pi_2} L_c(\pi_{1, \star}, \pi_2) + \Ecal_b(c) \explain{Since $\wh{\pi}_1 = \argmax_{\pi_1} \min_{\pi_2} L_c$} \\
 &\le r_c(\pi_{1, \star}, \wh{\pi}_2) - L_c(\pi_{1, \star}, \wh{\pi}_2) + \Ecal_b(c) \explain{Where $\wh{\pi}_2 = \argmin_{\pi_2} L_c(\pi_{1, \star}, \pi_2)$} \\
 &\le \EE_{s, a \sim d^{\pi_{1, \star}, \wh{\pi}_2} \mid c} \left[ 2\beta(s, a) \right] + 2\Ecal_b(c) \explain{By \eqref{eq: approximate pessimism decomposition}} \\
 &\explaind{\le \EE_{s, a \sim d^{\pi_{1, \star}, \wh{\pi}_2} \mid c} \left[ (2 + 2\kappa_b)\beta(s, a) \right].}{Since $\Ecal_b(c) \le \kappa_b \beta(s, a)$ point-wise} \label{eq:contextual_gap_bound}
\end{align}

To bound the global gap, we integrate over the public context marginal $c \sim \rho_\Ccal(c)$:
\begin{align}
 Gap(\wh{\pi}_1) &\le \EE_{c \sim \rho_\Ccal} \left[ \EE_{s, a \sim d^{\pi_{1, \star}, \wh{\pi}_2} \mid c} \left[ (2 + 2\kappa_b)\beta(s, a) \right] \right] \\
 &= (2 + 2\kappa_b) \sum_{s \in \Scal, a \in \Acal} d^{\pi_{1, \star}, \wh{\pi}_2}(s, a) \beta(s, a). \label{eq: global expectation}
\end{align}

Define the count concentration event $\Ecal_{\text{count}} \defeq \{ \forall (s, a) \in \mathrm{supp}(d^\rho): N(s, a) \ge \frac{1}{2} n d^\rho(s, a) \}$, which is the event actually established below and, by Convention~1, the only one the bound requires. By the multiplicative Chernoff bound, for a specific $(s, a)$ such that $d^\rho(s, a) > 0$:
\begin{equation}
 \P\left( N(s, a) < \frac{1}{2} n d^\rho(s, a) \right) \le \exp\left( -\frac{n d^\rho(s, a)}{8} \right).
\end{equation}
Applying a union bound over all $(s, a) \in \Scal \times \Acal$ in the support of $d^\rho$:
\begin{equation}
 \P(\Ecal_{\text{count}}^c) \le |\Scal| |\Acal| \exp\left( -\frac{n d^\rho_{\min}}{8} \right).
\end{equation}
To ensure $\P(\Ecal_{\text{count}}) \ge 1 - \delta/2$, we require $n \ge \frac{8}{d^\rho_{\min}} \log\left( \frac{2 |\Scal| |\Acal|}{\delta} \right)$. Under $\Ecal_{\text{count}},$ the uncertainty bonus satisfies $\beta(s, a)^2 \le \frac{4 \iota}{n d^\rho(s, a)}$. Substituting this into the global gap \eqref{eq: global expectation} and applying the Cauchy-Schwarz inequality:
\begin{align}
 &\quad Gap(\wh{\pi}_1)\\
 &\le (2 + 2\kappa_b) \sqrt{ \sum_{s, a} d^{\pi_{1, \star}, \wh{\pi}_2}(s, a) } \sqrt{ \sum_{s, a} d^{\pi_{1, \star}, \wh{\pi}_2}(s, a) \beta(s, a)^2 } \explain{Cauchy-Schwarz} \\
 &\le (2 + 2\kappa_b) \sqrt{1} \cdot \sqrt{ \sum_{s, a} d^{\pi_{1, \star}, \wh{\pi}_2}(s, a) \frac{4 \iota}{n d^\rho(s, a)} } \explain{Conditioned on $\Ecal_{\text{count}}$} \\
 &= 2(2 + 2\kappa_b) \sqrt{ \frac{\iota}{n} \sum_{s, a} \frac{d^{\pi_{1, \star}, \wh{\pi}_2}(s, a)}{d^\rho(s, a)} } \\
 &\explaind{\le 2(2 + 2\kappa_b) \sqrt{ \frac{C^\star_{\text{Pres}} |\Scal| |\Acal| \iota}{n} }}{By Assumption~\ref{ass:unilateral_coverage}} \label{eq: global gap through CS}
\end{align}
The final bound establishes the $\tilde{O} \left( (1 + \kappa_b) \sqrt{\frac{C^\star_{\text{Pres}} |\Scal| |\Acal|}{n}} \right)$ rate. 
\end{proof}
\subsection{Proof of Theorem~\ref{thm:iterative_gap}}
\label{proof:iterative_gap}
\begin{proof}
We first establish the local regret guarantee for the Mirror Descent update using the Kullback-Leibler (KL) divergence as a potential function.

\begin{lemma}[Local MWU Regret]
\label{lem:mwu_regret}
Let $Q_t \in [-1, 1]$ for all $t \in [T]$ and assume $|\Acal| \ge 2$, the case $|\Acal| = 1$ being immediate since the regret is then identically zero. For any fixed private state $x$ and any fixed comparator distribution $\pi_\star(\cdot \mid x)$, the update rule $\pi_{t+1}(a \mid x) \propto \pi_t(a \mid x) \exp(\eta Q_t(x, a))$ with uniform initialization $\pi_1(a \mid x) = 1/|\Acal|$ and learning rate $\eta > 0$ guarantees:

\begin{equation}
 \frac{1}{T} \sum_{t=1}^T \sum_a \pi_\star(a \mid x) Q_t(x, a) - \frac{1}{T} \sum_{t=1}^T \sum_a \pi_t(a \mid x) Q_t(x, a) \le \frac{\log |\Acal|}{\eta T} + \frac{\eta}{2}.
\end{equation}

\end{lemma}

\begin{proof}[Proof of Lemma \ref{lem:mwu_regret}]
 Let $\Phi_t \defeq \KL(\pi_\star(\cdot \mid x) \,\|\, \pi_t(\cdot \mid x))$. Since $\pi_{t+1}(a \mid x) = \pi_t(a \mid x)\, e^{\eta Q_t(x, a)} / Z_t$ with $Z_t \defeq \sum_{a'} \pi_t(a' \mid x)\, e^{\eta Q_t(x, a')}$,
\begin{equation*}
\Phi_t - \Phi_{t+1} = \sum_a \pi_\star(a \mid x) \log \frac{\pi_{t+1}(a \mid x)}{\pi_t(a \mid x)} = \eta \sum_a \pi_\star(a \mid x)\, Q_t(x, a) - \log Z_t .
\end{equation*}
Because $Q_t(x, \cdot) \in [-1, 1]$, Hoeffding's lemma applied to $Q_t(x, A_t)$ with $A_t \sim \pi_t(\cdot \mid x)$ gives $\log Z_t \le \eta \sum_a \pi_t(a \mid x)\, Q_t(x, a) + \eta^2/2$ for every $\eta > 0$, so
\begin{equation*}
\Phi_t - \Phi_{t+1} \ge \eta \sum_a \big( \pi_\star(a \mid x) - \pi_t(a \mid x) \big) Q_t(x, a) - \frac{\eta^2}{2}.
\end{equation*}
Summing over $t \in [T]$ and using $\Phi_{T+1} \ge 0$ and, for the uniform initialization, $\Phi_1 = \log|\Acal| + \sum_a \pi_\star(a \mid x) \log \pi_\star(a \mid x) \le \log |\Acal|$, we obtain $\eta \sum_{t=1}^T \sum_a (\pi_\star(a \mid x) - \pi_t(a \mid x))\, Q_t(x, a) \le \log |\Acal| + \eta^2 T/2$. Dividing by $\eta T$ proves the claim.

\end{proof}

Let $\wh{V}(c) \defeq \max_{\pi_1} \min_{\pi_2} L_c(\pi_1, \pi_2)$ be the exact idealized minimax value. Since the clipped rewards satisfy $\underline{r} \in [-1, 1],$ the conditional expectations $Q_{1,t}$ and $Q_{2,t}$ are bounded in $[-1, 1]$.

By applying Lemma \ref{lem:mwu_regret} to Player 1 at each private state $x^1 \in \Xcal_1$ using the shared algorithm learning rate $\eta = \sqrt{\log(\max\{|\Acal_1|, |\Acal_2|\}) / T},$ the right-hand side of the lemma evaluates to:

\begin{equation}
 \frac{\log |\Acal_1|}{\eta T} + \frac{\eta}{2} \le \frac{\log(\max\{|\Acal_1|, |\Acal_2|\})}{\eta T} + \frac{\eta}{2} = \frac{3}{2}\sqrt{\frac{\log(\max\{|\Acal_1|, |\Acal_2|\})}{T}} \le 2\sqrt{\frac{\log(\max\{|\Acal_1|, |\Acal_2|\})}{T}}.
\end{equation}

We lift these local bounds to the global objective by taking the expectation over the marginal belief $\wh{b}(x^1 \mid c):$
\begin{equation}
 \sum_{x^1} \wh{b}(x^1 \mid c) \left[ \frac{1}{T} \sum_{t=1}^T \sum_{a^1} (\pi_1(a^1 \mid c, x^1) - \pi_{1,t}(a^1 \mid c, x^1)) Q_{1,t}(x^1, a^1) \right] \le 2\sqrt{\frac{\log(\max\{|\Acal_1|, |\Acal_2|\})}{T}}. \label{eq: lifted_regret_p1}
\end{equation}
 By the definition of $Q_{1,t}$ and $\wh{b}(x^1 \mid c)\, \wh{b}(x^2 \mid c, x^1) = \wh{b}(x^1, x^2 \mid c)$, we have $\sum_{x^1, a^1} \wh{b}(x^1 \mid c)\, \pi_1(a^1 \mid c, x^1)\, Q_{1,t}(x^1, a^1) = L_c(\pi_1, \pi_{2,t})$ for every $\pi_1 \in \Pi_1$; applying this to $\pi_1$ and to $\pi_{1,t}$ in \eqref{eq: lifted_regret_p1} gives the global regret bound for Player~1:
\begin{equation}
 \frac{1}{T} \sum_{t=1}^T L_c(\pi_1, \pi_{2,t}) - \frac{1}{T} \sum_{t=1}^T L_c(\pi_{1,t}, \pi_{2,t}) \le 2\sqrt{\frac{\log(\max\{|\Acal_1|, |\Acal_2|\})}{T}} \defeq \epsilon_1. \label{eq:regret_p1}
\end{equation}
For Player 2, who seeks to minimize the reward (maximize cost), Lemma \ref{lem:mwu_regret} is applied to the sequence of negative local $Q$-values $-Q_{2,t}(x^2, a^2)$. For any fixed comparator $\pi_2 \in \Pi_2,$ the local cost-regret bound at each state $x^2 \in \Xcal_2$ using the shared learning rate $\eta$ is:
\begin{equation}
 \frac{1}{T} \sum_{t=1}^T \sum_{a^2} (\pi_2(a^2 \mid c, x^2) - \pi_{2,t}(a^2 \mid c, x^2)) (-Q_{2,t}(x^2, a^2)) \le 2\sqrt{\frac{\log(\max\{|\Acal_1|, |\Acal_2|\})}{T}}.
\end{equation}
Distributing the negative sign into the policy difference and taking the expectation over the marginal belief $\wh{b}(x^2 \mid c)$ yields:
\begin{equation}
 \sum_{x^2} \wh{b}(x^2 \mid c) \left[ \frac{1}{T} \sum_{t=1}^T \sum_{a^2} (\pi_{2,t}(a^2 \mid c, x^2) - \pi_2(a^2 \mid c, x^2)) Q_{2,t}(x^2, a^2) \right] \le 2\sqrt{\frac{\log(\max\{|\Acal_1|, |\Acal_2|\})}{T}}.
\end{equation}
 In the same way, $\sum_{x^2, a^2} \wh{b}(x^2 \mid c)\, \pi_2(a^2 \mid c, x^2)\, Q_{2,t}(x^2, a^2) = L_c(\pi_{1,t}, \pi_2)$ for every $\pi_2 \in \Pi_2$, including $\pi_2 = \pi_{2,t}$, which gives the second global regret guarantee:
\begin{equation}
 \frac{1}{T} \sum_{t=1}^T L_c(\pi_{1,t}, \pi_{2,t}) - \frac{1}{T} \sum_{t=1}^T L_c(\pi_{1,t}, \pi_2) \le 2\sqrt{\frac{\log(\max\{|\Acal_1|, |\Acal_2|\})}{T}} \defeq \epsilon_2. \label{eq:regret_p2}
\end{equation}

 Since $L_c$ is bilinear, $\frac{1}{T}\sum_{t} L_c(\pi_1, \pi_{2,t}) = L_c(\pi_1, \bar{\pi}_2)$ and $\frac{1}{T}\sum_{t} L_c(\pi_{1,t}, \pi_2) = L_c(\bar{\pi}_1, \pi_2)$ for all $\pi_1 \in \Pi_1$ and $\pi_2 \in \Pi_2$, where $\bar{\pi}_i \defeq \frac{1}{T}\sum_{t} \pi_{i,t}$. Since \eqref{eq:regret_p1} and \eqref{eq:regret_p2} hold for any fixed policies, it follows that:
\begin{equation}
 L_c(\pi_1, \bar{\pi}_2) - \epsilon_1 \le \frac{1}{T} \sum_{t=1}^T L_c(\pi_{1,t}, \pi_{2,t}) \le L_c(\bar{\pi}_1, \pi_2) + \epsilon_2, \quad \forall \pi_1 \in \Pi_1, \pi_2 \in \Pi_2.
\end{equation}
Since this inequality holds for the entire policy space, we can take the supremum over $\pi_1$ on the left and the infimum over $\pi_2$ on the right to obtain the saddle-point optimality bound:
\begin{equation}
 \max_{\pi_1 \in \Pi_1} L_c(\pi_1, \bar{\pi}_2) \le \min_{\pi_2 \in \Pi_2} L_c(\bar{\pi}_1, \pi_2) + \epsilon_1 + \epsilon_2. \label{eq:mwu_saddle}
\end{equation}
By the bilinearity of $L_c$ and the fact that the policy spaces $\Pi_1(c)$ and $\Pi_2(c)$ are compact and convex, the conditions for Sion's Minimax Theorem \citep{sion1958general} are satisfied. Thus, there exists a unique value for the empirical game:
\begin{equation}
 \wh{V}(c) \defeq \min_{\pi_2 \in \Pi_2} \max_{\pi_1 \in \Pi_1} L_c(\pi_1, \pi_2). \label{def: hat V}
\end{equation}
From the definition of the minimum, it holds for the specific average policy $\bar{\pi}_2 \in \Pi_2$ that:
\begin{equation}
 \wh{V}(c) \le \max_{\pi_1 \in \Pi_1} L_c(\pi_1, \bar{\pi}_2). \label{eq:value_upper_bound}
\end{equation}
Substituting the bound \eqref{eq:value_upper_bound} into the left-hand side of \eqref{eq:mwu_saddle} yields the required relationship between the empirical value and the performance of our iterative policy:
\begin{equation}
 \wh{V}(c) \le \min_{\pi_2 \in \Pi_2} L_c(\bar{\pi}_1, \pi_2) + \epsilon_1 + \epsilon_2.
\end{equation}
Rearranging this inequality isolates the optimization error relative to the minimax value:
\begin{equation}
 -\min_{\pi_2 \in \Pi_2} L_c(\bar{\pi}_1, \pi_2) \le -\wh{V}(c) + \epsilon_1 + \epsilon_2. \label{eq:pessimistic_opt_error}
\end{equation}
Finally, we substitute \eqref{eq:pessimistic_opt_error} into the contextual exploitability gap expansion from \eqref{eq:contextual_gap_bound}. Here and for the remainder of this proof, $\nu^\star_c \in \argmin_{\pi_2 \in \Pi_2} r_c(\bar{\pi}_1, \pi_2)$ denotes the best response against the iterative policy $\bar{\pi}_1$, so that $r_c(\bar{\pi}_1, \nu^\star_c) = \min_{\pi_2 \in \Pi_2} r_c(\bar{\pi}_1, \pi_2)$.

\begin{align}
 V^\star(c) - r_c(\bar{\pi}_1, \nu^\star_c) &\le \min_{\pi_2 \in \Pi_2} r_c(\pi_{1, \star}, \pi_2) - \min_{\pi_2 \in \Pi_2} L_c(\bar{\pi}_1, \pi_2) + \Ecal_b(c) \\
 &\le \min_{\pi_2 \in \Pi_2} r_c(\pi_{1, \star}, \pi_2) - \wh{V}(c) + \Ecal_b(c) + \epsilon_1 + \epsilon_2.
\end{align}
Since $\wh{V}(c) = \max_{\pi_1} \min_{\pi_2} L_c(\pi_1, \pi_2) \ge \min_{\pi_2} L_c(\pi_{1, \star}, \pi_2),$ we have:
\begin{align}
 V^\star(c) - r_c(\bar{\pi}_1, \nu^\star_c) &\le \min_{\pi_2 \in \Pi_2} r_c(\pi_{1, \star}, \pi_2) - \min_{\pi_2 \in \Pi_2} L_c(\pi_{1, \star}, \pi_2) + \Ecal_b(c) + \epsilon_1 + \epsilon_2 \\
 &\le \EE_{s, a \sim d^{\pi_{1, \star}, \wh{\pi}_2} \mid c} \left[ (2 + 2\kappa_b)\beta(s, a) \right] + \epsilon_1 + \epsilon_2,
\end{align}
where the final inequality follows from the statistical pessimism derivation in \eqref{eq:contextual_gap_bound}.

To bound the global gap, we integrate over the public context distribution $c \sim \rho_\Ccal.$ Because $\epsilon_1$ and $\epsilon_2$ are deterministic optimization constants independent of the context, the expectation perfectly splits:
\begin{align}
 Gap(\bar{\pi}_1) &= \EE_{c \sim \rho_\Ccal} \left[ V^\star(c) - r_c(\bar{\pi}_1, \nu^\star_c) \right] \\
 &\le \EE_{c \sim \rho_\Ccal} \left[ \EE_{s, a \sim d^{\pi_{1, \star}, \wh{\pi}_2} \mid c} \left[ (2 + 2\kappa_b)\beta(s, a) \right] \right] + \epsilon_1 + \epsilon_2.
\end{align}
Applying the global Cauchy-Schwarz integration derived in Equation \eqref{eq: global gap through CS} to bound the statistical expectation term, we obtain:
\begin{equation}
 Gap(\bar{\pi}_1) \le 2(2 + 2\kappa_b) \sqrt{\frac{C^\star_{\text{Pres}} |\Scal| |\Acal| \iota}{n}} + \epsilon_1 + \epsilon_2. \label{eq: gap with epsilon}
\end{equation}
To complete the proof, we substitute the explicitly symmetric optimization errors defined in \eqref{eq:regret_p1} and \eqref{eq:regret_p2}. Because we utilized the shared learning rate bounded by the maximum action space, $\epsilon_1$ and $\epsilon_2$ are identical, yielding an exact sum:
\begin{equation}
 \epsilon_1 + \epsilon_2 = 4\sqrt{\frac{\log(\max\{|\Acal_1|, |\Acal_2|\})}{T}}.\label{eq: epsilon sum}
\end{equation}
Substituting this sum into \eqref{eq: gap with epsilon} yields the final exploitability guarantee:
\begin{equation}
 Gap(\bar{\pi}_1) \le 2(2 + 2\kappa_b) \sqrt{\frac{C^\star_{\text{Pres}} |\Scal| |\Acal| \iota}{n}} + 4\sqrt{\frac{\log(\max\{|\Acal_1|, |\Acal_2|\})}{T}}.
\end{equation}

This establishes the $\tilde{O}(1/\sqrt{n} + 1/\sqrt{T})$ convergence rate for the iterative asymmetric bandit solver.

\end{proof}

\subsection{Parameterized PPA-PMD Implementation}
\label{app:fa_discussion}

 The procedure below instantiates the framework of Section~\ref{sec:function_approximation} with parameterised actors and minibatched contexts.

\begin{algorithm}[t]
\caption{Public-Private Asymmetric Policy Mirror Descent (PPA-PMD)}
\label{alg:paac}
\begin{algorithmic}[1]
\Require Offline dataset $\mathcal{D}$, function classes $\mathcal{F}$ and $\mathcal{B}$, uncertainty tolerance $\beta$, iterations $T$, step size $\eta > 0$, and initial parameters $\theta_{1,1}, \theta_{2,1}$ whose policies have full support.
\State \textbf{Phase 1: Offline Supervised Estimation (The Critic)}
\State Estimate belief $\wh{b} \in \mathcal{B}$ via MLE \Comment{Eq. \eqref{eq:belief_mle}}
\State Fit empirical reward $\wh{r} \in \mathcal{F}$ and construct $\underline{r}(s, a) = \inf_{f \in \mathcal{F}(\beta)} f(s, a)$ \Comment{Eq. \eqref{eq:version_space_pessimism}}
\State \textbf{Phase 2: Decentralized Actor Optimization}
\For {$t = 1, \dots, T$}
 \State Sample mini-batch of public contexts $c \sim \mathcal{D}$.
 \State \textbf{Critic Evaluation:} Compute local $Q$-values $Q_{1,t}$ and $Q_{2,t}$ \Comment{Eqs. \eqref{eq:fa_q1}, \eqref{eq:fa_q2}}
 \State \textbf{Policy Mirror Descent Update:} expectations over the mini-batch contexts $c$ and $x^i \sim \wh{b}(x^i \mid c)$
 \State \quad $\theta_{1, t+1} \leftarrow \argmin_{\theta_1} \mathbb{E}_{c, x^1} [ -\eta \langle \pi_{\theta_1}, Q_{1,t} \rangle + \text{KL}(\pi_{\theta_1} \| \pi_{\theta_{1,t}}) ]$
 \State \quad $\theta_{2, t+1} \leftarrow \argmin_{\theta_2} \mathbb{E}_{c, x^2} [ \eta \langle \pi_{\theta_2}, Q_{2,t} \rangle + \text{KL}(\pi_{\theta_2} \| \pi_{\theta_{2,t}}) ]$
\EndFor
\Ensure Average strategy for Player 1: $\bar{\pi}_1 = \frac{1}{T} \sum_{t=1}^T \pi_{\theta_{1,t}}.$
\end{algorithmic}
\end{algorithm} 

 Algorithm~\ref{alg:paac} is a parameterized minibatch implementation of the framework; Theorem~\ref{thm:paac_gap} applies when its actor updates satisfy the external-regret condition of Assumption~\ref{ass:actor_regret}.

Our analysis abstracts the actor update; complementary convergence theory for one- and two-timescale stochastic approximation under Markovian noise is developed in \citep{liu2025ode,mahadevan2026convergence}. 

More broadly, Transformers have been studied both as decision policies for high-dimensional operational problems and as models capable of implementing RL-style updates in context \citep{liu2026or,xie2026beyond}.

\subsection{Proof of Theorem~\ref{thm:paac_gap}}
\label{proof:paac_gap}
\begin{proof}

We first establish the standard statistical guarantees for our offline estimation phase.

\begin{lemma}[Offline Estimation Guarantees]
\label{lem:offline_estimation}
Under Assumption \ref{ass:realizability}, setting the version space tolerance $\beta = \frac{136 \log(2|\Fcal|/\delta)}{n},$ with probability at least $1-\delta,$ the true reward resides in the version space $r \in \Fcal(\beta),$ and for all $f \in \Fcal(\beta):$

\begin{equation}
 \EE_{s, a \sim d^\rho} \left[ (f(s, a) - r(s, a))^2 \right] \le 2\beta.
\end{equation}
Furthermore, the Maximum Likelihood Estimate $\wh{b} \in \Bcal$ guarantees an expected $L_1$ error bounded by $\Ecal_\Bcal \defeq 2\sqrt{2 \log(2|\Bcal|/\delta) / n}$ such that:
\begin{equation}
 \EE_{c \sim \rho_\Ccal} \left[ \| b(\cdot \mid c) - \wh{b}(\cdot \mid c) \|_1 \right] \le \Ecal_\Bcal.
\end{equation}
\end{lemma}
\begin{proof}[Proof of Lemma \ref{lem:offline_estimation}]
The proof proceeds in two parts, allocating a failure probability of $\delta/2$ to the reward regression and $\delta/2$ to the belief estimation, yielding a global guarantee with probability at least $1-\delta$ via a union bound.

\paragraph{Part 1: Reward estimation.}
This part establishes $\sup_{f \in \Fcal(\beta)} \EE_{s, a \sim d^\rho}[(f-r)^2] \le 2\beta$ on the population behavior measure, which together with Assumption~\ref{ass:fa_concentrability} will yield the target-distribution bound in the main proof body. A uniform fast-rate ERM bound follows from a Bernstein concentration combined with a union bound over $\Fcal$; realizability and the version-space tolerance $\beta$ then convert the empirical bound into a true-risk guarantee. No concentrability assumption enters this part.

Let $\Lcal(f) = \EE_{s, a \sim d^\rho}[(f(s, a) - r(s, a))^2]$ denote the true expected squared error under the population behavior measure, and $\Lcal_\Dcal(f) = \frac{1}{n} \sum_{k=1}^n (f(s^{(k)}, a^{(k)}) - r^{(k)})^2$ denote the empirical risk.
Consider the random variable $X = (f(s,a) - r(s,a))^2\text{ with } (s,a) \sim d^\rho.$ Since functions in $\Fcal$ and the true reward $r$ are bounded in $[-1, 1],$ we have $X \in [0, 4].$ The expectation is $\EE[X] = \Lcal(f).$ The variance is bounded by the uncentered second moment:
\begin{equation}
 \V(X) \le \EE[X^2] = \EE[(f(s,a) - r(s,a))^4] \le 4 \EE[(f(s,a) - r(s,a))^2] = 4\Lcal(f).
\end{equation}
Applying Bernstein's inequality for a fixed $f \in \Fcal:$
\begin{equation}
 \PP\left( \Lcal(f) - \Lcal_\Dcal(f) \ge t \right) \le \exp\left( \frac{-n t^2 / 2}{\V(X) + 4t/3} \right) \le \exp\left( \frac{-n t^2 / 2}{4\Lcal(f) + 4t/3} \right).
\end{equation}

Equating the right-hand side to $\frac{\delta}{2|\Fcal|}$ and letting $\iota_\Fcal \defeq \log(2|\Fcal|/\delta),$ we solve the quadratic $\frac{n t^2}{2} - \frac{4}{3} t \iota_\Fcal - 4 \Lcal(f) \iota_\Fcal = 0$ for its positive root and apply $\sqrt{u + v} \le \sqrt{u} + \sqrt{v}$:
\begin{equation}
 t = \frac{1}{n}\left[ \frac{4}{3}\iota_\Fcal + \sqrt{ \frac{16}{9}\iota_\Fcal^2 + 8 n \Lcal(f) \iota_\Fcal } \right] \le \frac{8\iota_\Fcal}{3n} + \sqrt{\frac{8\Lcal(f)\iota_\Fcal}{n}}.
\end{equation}
Using the AM-GM inequality variant $\sqrt{xy} \le \frac{x}{4} + y$ with $x = \Lcal(f)$ and $y = \frac{8\iota_\Fcal}{n},$ we bound the square root term:
\begin{equation}
 \sqrt{ \Lcal(f) \cdot \frac{8\iota_\Fcal}{n} } \le \frac{1}{4}\Lcal(f) + \frac{8\iota_\Fcal}{n}.
\end{equation}
Substituting this back yields the deviation bound:
\begin{equation}
 t \le \frac{1}{4}\Lcal(f) + \left(\frac{8}{3} + 8\right)\frac{\iota_\Fcal}{n} = \frac{1}{4}\Lcal(f) + \frac{32\iota_\Fcal}{3n}.
\end{equation}
By applying a union bound over all $f \in \Fcal,$ it holds uniformly with probability at least $1-\delta/2$ that:
\begin{align}
 \Lcal(f) - \Lcal_\Dcal(f) &\le \frac{1}{4}\Lcal(f) + \frac{32\iota_\Fcal}{3n} \\
 \frac{3}{4}\Lcal(f) &\le \Lcal_\Dcal(f) + \frac{32\iota_\Fcal}{3n} \\
 \Lcal(f) &\le \frac{4}{3}\Lcal_\Dcal(f) + \frac{128\iota_\Fcal}{9n}. \label{eq:fast_rate_erm}
\end{align}

Under Assumption \ref{ass:realizability}, the true reward satisfies $r \in \Fcal.$ Because the environment rewards are deterministic, $r(s^{(k)}, a^{(k)}) = r^{(k)}$ identically, meaning the empirical risk of the true reward is exactly zero ($\Lcal_\Dcal(r) = 0$).
Since $\wh{r} = \argmin_{f \in \Fcal} \Lcal_\Dcal(f),$ it must be that $\Lcal_\Dcal(\wh{r}) \le \Lcal_\Dcal(r) = 0,$ implying $\Lcal_\Dcal(\wh{r}) = 0.$

The version space is defined as $\Fcal(\beta) = \{f \in \Fcal: \Lcal_\Dcal(f) - \Lcal_\Dcal(\wh{r}) \le \beta\}.$ Since $\Lcal_\Dcal(r) - \Lcal_\Dcal(\wh{r}) = 0 \le \beta,$ it trivially holds that $r \in \Fcal(\beta)$ for any $\beta \ge 0.$
For any arbitrary $f \in \Fcal(\beta),$ we have $\Lcal_\Dcal(f) \le \beta.$ Substituting this into \eqref{eq:fast_rate_erm}:
\begin{equation}
 \Lcal(f) \le \frac{4}{3}\beta + \frac{128\iota_\Fcal}{9n}.
\end{equation}

Given the lemma's specification of $\beta = \frac{136 \iota_\Fcal}{n},$ we observe that $\frac{128\iota_\Fcal}{9n}$ $\le \frac{2}{3} \left( \frac{136\iota_\Fcal}{n} \right) = \frac{2}{3}\beta .$ Therefore:
\begin{equation}
 \Lcal(f) \le \frac{4}{3}\beta + \frac{2}{3}\beta = 2\beta,
\end{equation}
which satisfies the required risk guarantee for the reward estimator.

\paragraph{Part 2: Belief estimation.}
The empirical belief is obtained via MLE \eqref{eq:belief_mle}: $\wh{b} = \argmax_{b \in \Bcal} \frac{1}{n} \sum_{k=1}^n \log b(x^{1, (k)}, x^{2, (k)} \mid c^{(k)}).$

 We bound the MLE error through the squared Hellinger distance. Introduce the joint distributions on $\Ccal \times \Xcal_1 \times \Xcal_2$
\begin{equation*}
P_b(c, x) \defeq \rho_\Ccal(c)\, b(x \mid c), \qquad P_q(c, x) \defeq \rho_\Ccal(c)\, q(x \mid c) \quad \text{for } q \in \Bcal,
\end{equation*}
and write $H^2(P, Q) \defeq 1 - \sum_z \sqrt{P(z) Q(z)}$ for the squared Hellinger distance. The $n$ observations $(c^{(k)}, x^{1,(k)}, x^{2,(k)})$ are i.i.d.\ draws from $P_b$, and the MLE \eqref{eq:belief_mle} maximises the joint likelihood over $\Bcal$ because the factor $\rho_\Ccal(c)$ is common to all candidates.

\textbf{Step 1: pointwise likelihood-ratio bound.} Fix $q \in \Bcal$ and let $L_n(q) = \prod_{k=1}^n q(x^{(k)} \mid c^{(k)})$. By Markov's inequality applied to the square-root likelihood ratio, and then independence,
\begin{equation}
\PP_b\left( L_n(q) \ge L_n(b) \right) = \PP_b\left( \sqrt{L_n(q)/L_n(b)} \ge 1 \right) \le \EE_b\left[ \sqrt{L_n(q)/L_n(b)} \right] = \left( \sum_z \sqrt{P_b(z) P_q(z)} \right)^{n},
\end{equation}
and since $\sum_z \sqrt{P_b P_q} = 1 - H^2(P_b, P_q) \le \exp(-H^2(P_b, P_q))$, the right-hand side is at most $\exp\left( -n H^2(P_b, P_q) \right)$.

\textbf{Step 2: union bound.} Under Assumption~\ref{ass:realizability} we have $b \in \Bcal$, so the maximiser $\wh{b}$ satisfies $L_n(\wh{b}) \ge L_n(b)$ by definition. Applying Step~1 to every $q \in \Bcal$ with $H^2(P_b, P_q) > \log(2|\Bcal|/\delta)/n$ and taking a union bound, with probability at least $1 - \delta/2$ no such candidate can beat $b$, whence
\begin{equation}
H^2(P_b, P_{\wh{b}}) \le \frac{\log(2|\Bcal|/\delta)}{n}.
\end{equation}

\textbf{Step 3: conversion to $L_1$.} Because $P_b$ and $P_{\wh{b}}$ share the public-context marginal $\rho_\Ccal$,
\begin{equation}
\EE_{c \sim \rho_\Ccal} \left[ \| b(\cdot \mid c) - \wh{b}(\cdot \mid c) \|_1 \right] = \sum_{c, x} \left| P_b(c, x) - P_{\wh{b}}(c, x) \right| = \| P_b - P_{\wh{b}} \|_1.
\end{equation}
The standard comparison $\|P - Q\|_1 = 2 \, d_{\mathrm{TV}}(P, Q) \le 2\sqrt{2} \, H(P, Q)$ then gives
\begin{equation}
\EE_{c \sim \rho_\Ccal} \left[ \| b(\cdot \mid c) - \wh{b}(\cdot \mid c) \|_1 \right] \le 2\sqrt{2 H^2(P_b, P_{\wh{b}})} \le 2 \sqrt{\frac{2 \log(2|\Bcal|/\delta)}{n}} \defeq \Ecal_\Bcal.
\end{equation}

Combining Part 1 and Part 2 via a union bound, both guarantees hold simultaneously with probability at least $1-\delta,$ completing the proof.
\end{proof}

Now we prove Theorem~\ref{thm:paac_gap}.
We begin by establishing the approximate pessimism of the FA objective $L_c$ relative to the true reward $r_c$. By Lemma \ref{lem:offline_estimation}, $r \in \Fcal(\beta).$ Therefore, the constructed critic $\underline{r}(s, a) = \inf_{f \in \Fcal(\beta)} f(s, a)$ satisfies $\underline{r}(s, a) \le r(s, a).$ The point-wise underestimation error is bounded by the version space diameter:
\begin{equation}
 r(s, a) - \underline{r}(s, a) \le \sup_{f \in \Fcal(\beta)} f(s, a) - \inf_{f \in \Fcal(\beta)} f(s, a) \le 2 \sup_{f \in \Fcal(\beta)} |f(s, a) - r(s, a)|. \label{eq:fa_pessimism_bound}
\end{equation}
Let $\nu^\star_c$ denote the true best-response policy of Player 2 against the learner's iterative strategy $\bar{\pi}_1,$ defined as:
\begin{equation}
 \nu^\star_c \in \argmin_{\pi_2 \in \Pi_2} r_c(\bar{\pi}_1, \pi_2).
\end{equation}
Under this definition, the contextual exploitability gap at context $c$ is $V^\star(c) - r_c(\bar{\pi}_1, \nu^\star_c).$

To bound the exploitability gap of the iterative policy $\bar{\pi}_1,$ we first decompose it by adding and subtracting the empirical minimax value $\wh{V}(c)$ and the empirical performance of the iterative policy:
\begin{align}
 V^\star(c) - r_c(\bar{\pi}_1, \nu^\star_c) &= \underbrace{V^\star(c) - \wh{V}(c)}_{\text{Value Error}} + \underbrace{\wh{V}(c) - \min_{\pi_2 \in \Pi_2} L_c(\bar{\pi}_1, \pi_2)}_{\text{Optimization Error}} + \underbrace{\min_{\pi_2 \in \Pi_2} L_c(\bar{\pi}_1, \pi_2) - r_c(\bar{\pi}_1, \nu^\star_c)}_{\text{Pessimism Error}}.
 \label{eq: dual gap decomp three}
\end{align}
We bound each term individually. First, for the \textbf{Value Error}, recall that $\wh{V}(c) = \max_{\pi_1} \min_{\pi_2} L_c(\pi_1, \pi_2)$ \eqref{def: hat V}. By the property of the maximum, $\wh{V}(c) \ge \min_{\pi_2} L_c(\pi_{1, \star}, \pi_2).$ Thus:
\begin{align}
 V^\star(c) - \wh{V}(c) \le \min_{\pi_2 \in \Pi_2} r_c(\pi_{1, \star}, \pi_2) - \min_{\pi_2 \in \Pi_2} L_c(\pi_{1, \star}, \pi_2). \label{eq:telescope_1}
\end{align}
Second, the \textbf{Optimization Error} is bounded by the no-regret property of the iterative solver derived in \eqref{eq:pessimistic_opt_error}:
\begin{align}
 \wh{V}(c) - \min_{\pi_2 \in \Pi_2} L_c(\bar{\pi}_1, \pi_2) \le \epsilon_1 + \epsilon_2. \label{eq:telescope_2}
\end{align}

Third, for \textbf{Pessimism Error}, we evaluate the consistency error of the iterative policy $\bar{\pi}_1.$ By the definition of the best-response $\nu^\star_c = \argmin_{\pi_2} r_c(\bar{\pi}_1, \pi_2),$ it follows that:
\begin{align}
 \min_{\pi_2 \in \Pi_2} L_c(\bar{\pi}_1, \pi_2) - r_c(\bar{\pi}_1, \nu^\star_c) &\le \max_{\pi_2 \in \Pi_2} \left( L_c(\bar{\pi}_1, \pi_2) - r_c(\bar{\pi}_1, \pi_2) \right).
\end{align}
For any fixed $\pi_2,$ we decompose the difference into belief and reward estimation errors by adding and subtracting $\sum_{x, a} b(x \mid c) \underline{r}(s, a) \bar{\pi}_1 \pi_2:$
\begin{align}
 L_c(\bar{\pi}_1, \pi_2) - r_c(\bar{\pi}_1, \pi_2) &= \sum_{x, a} (\wh{b} - b) \underline{r} \bar{\pi}_1 \pi_2 + \sum_{x, a} b (\underline{r} - r) \bar{\pi}_1 \pi_2 \\
 &\le \sum_{x, a} |b - \wh{b}| \cdot |\underline{r}| \bar{\pi}_1 \pi_2 + 0 \explain{Since $\underline{r} \le r$ implies $(\underline{r} - r) \le 0$} \\
 &\le \sum_x |b(x \mid c) - \wh{b}(x \mid c)| \underbrace{\sum_a \bar{\pi}_1 \pi_2}_{=1} = \|b - \wh{b}\|_1.
\end{align}
Thus, the consistency error is bounded purely by the belief estimation accuracy:
\begin{align}
 \min_{\pi_2 \in \Pi_2} L_c(\bar{\pi}_1, \pi_2) - r_c(\bar{\pi}_1, \nu^\star_c) \le \|b(\cdot \mid c) - \wh{b}(\cdot \mid c)\|_1. \label{eq:telescope_3}
\end{align}
 Using the decomposition from \eqref{eq: approximate pessimism decomposition}, for any policy pair $(\pi_1, \pi_2)$ and context $c:$
\begin{align}
 & r_c(\pi_1, \pi_2) - L_c(\pi_1, \pi_2)\\
 =& \sum_{x, a} b (r - \underline{r}) \pi_1 \pi_2 + \sum_{x, a} (b - \wh{b}) \underline{r} \pi_1 \pi_2 \\
 \le& \EE_{s, a \sim d^{\pi_1, \pi_2} \mid c} \left[ 2 \sup_{f \in \Fcal(\beta)} |f(s, a) - r(s, a)| \right] + \| b(\cdot \mid c) - \wh{b}(\cdot \mid c) \|_1.\label{eq:fa_pessimism_bridge}
\end{align}

 The second term is at most $\|b(\cdot \mid c) - \wh{b}(\cdot \mid c)\|_1$ because $\|\underline{r}\|_\infty \le 1$, exactly as in \eqref{eq:telescope_3}. In the first term, $r(s, a) - \underline{r}(s, a) = \sup_{f \in \Fcal(\beta)} (r(s, a) - f(s, a)) \le \sup_{f \in \Fcal(\beta)} |f(s, a) - r(s, a)|$, which in particular implies the bound with $2 \sup_{f \in \Fcal(\beta)} |f - r|$ in \eqref{eq:fa_pessimism_bridge}.

Substituting the individual bounds for the three terms of \eqref{eq: dual gap decomp three} back into the contextual decomposition yields:
\begin{align}
 V^\star(c) - r_c(\bar{\pi}_1, \nu^\star_c) \le \EE_{s, a \sim d^{\pi_{1, \star}, \wh{\pi}_2} \mid c} \left[ 2 \sup_{f \in \Fcal(\beta)} |f(s, a) - r(s, a)| \right] + 2\| b(\cdot \mid c) - \wh{b}(\cdot \mid c) \|_1 + \epsilon_{\text{opt}},
\end{align}
where $\wh{\pi}_2 = \argmin_{\pi_2} L_c(\pi_{1, \star}, \pi_2).$
 Here $\epsilon_{\text{opt}} \defeq \epsilon_1 + \epsilon_2$, and Assumption~\ref{ass:actor_regret} gives the saddle-point optimization error $\epsilon_{\text{opt}} = (R_1(T) + R_2(T))/T$ in \eqref{eq:mwu_saddle}. For multiplicative weights, Lemma~\ref{lem:mwu_regret} gives $R_i(T) \le 2\sqrt{T \log(\max\{|\Acal_1|, |\Acal_2|\})}$, so $\epsilon_{\text{opt}} \le 4\sqrt{\log(\max\{|\Acal_1|, |\Acal_2|\})/T}$ as in \eqref{eq: epsilon sum}.

To reach the global exploitability guarantee, we integrate this contextual gap over the public context distribution $c \sim \rho_\Ccal:$
\begin{align}
 \text{Gap}(\bar{\pi}_1) &= \EE_{c \sim \rho_\Ccal} \left[ V^\star(c) - r_c(\bar{\pi}_1, \nu^\star_c) \right] \\
 &\le \EE_{c, s, a \sim d^{\pi_{1, \star}, \wh{\pi}_2}} \left[ 2 \sup_{f \in \Fcal(\beta)} |f(s, a) - r(s, a)| \right] + 2 \EE_{c \sim \rho_\Ccal} \| b - \wh{b} \|_1 + \epsilon_{\text{opt}}.
\end{align}
We now apply our concentration results.
For the belief error, Part~2 of Lemma~\ref{lem:offline_estimation} directly gives
\begin{equation}
 2 \EE_{c \sim \rho_\Ccal} \| b - \wh{b} \|_1 \le 2 \Ecal_\Bcal = 4 \sqrt{\frac{2 \log(2|\Bcal|/\delta)}{n}},
\end{equation}
since Part~2 bounds the $L_1$ error directly via the Hellinger distance.

 For the reward error, we apply Cauchy--Schwarz directly to the nonnegative random variable $\sup_{f \in \Fcal(\beta)}(f-r)^2$, keeping the supremum \emph{inside} the expectation; the uniform-ratio assumption (Assumption~\ref{ass:fa_concentrability}) then transfers the resulting expected squared error from $d^{\pi_{1,\star},\wh{\pi}_2}$ to the behavior distribution $d^\rho$ in a single step, after which Lemma~\ref{lem:offline_estimation} (Part~1) closes the chain:
\begin{align}
 \EE_{c, s, a \sim d^{\pi_{1, \star}, \wh{\pi}_2}} \left[ 2 \sup_{f \in \Fcal(\beta)} |f(s, a) - r(s, a)| \right]
 &\le 2 \sqrt{ \EE_{d^{\pi_{1,\star},\wh{\pi}_2}}\!\left[ \sup_{f \in \Fcal(\beta)} (f(s,a) - r(s,a))^2 \right] }
 \explain{Cauchy--Schwarz / Jensen} \\
 &\le 2 \sqrt{ \Cscr_{\Fcal} \cdot \sup_{f \in \Fcal(\beta)} \EE_{s, a \sim d^\rho}\!\left[ (f(s,a) - r(s,a))^2 \right] }
 \explain{Assumption~\ref{ass:fa_concentrability}} \\
 &\le 2 \sqrt{ 2\, \Cscr_{\Fcal}\, \beta }
 \explain{Lemma~\ref{lem:offline_estimation}, Part~1}.
\end{align}
Summing these integrated components and substituting $\beta = 136 \log(2|\Fcal|/\delta)/n$ from Lemma~\ref{lem:offline_estimation} (Part~1) and $\Ecal_\Bcal = 2\sqrt{2\log(2|\Bcal|/\delta)/n}$ completes the proof of Theorem \ref{thm:paac_gap}:

\begin{equation}
 \text{Gap}(\bar{\pi}_1) \le 2\sqrt{\frac{272\, \Cscr_{\Fcal} \log(2|\Fcal|/\delta)}{n}} + 4\sqrt{\frac{2\log(2|\Bcal|/\delta)}{n}} + \frac{R_1(T) + R_2(T)}{T}.
\end{equation}

With multiplicative weights the optimization term is at most $4\sqrt{\log(\max\{|\Acal_1|, |\Acal_2|\})/T}$.

\end{proof}

\newpage
\end{document}